\documentclass[lettersize,onecolumn]{IEEEtran} %,journal
\usepackage{amsmath,amssymb,amsfonts}
\usepackage{algorithm}
\usepackage{algorithmic}
\usepackage{array}
\usepackage[caption=false,font=normalsize,labelfont=sf,textfont=sf]{subfig}
\usepackage{textcomp}
\usepackage{stfloats}
\usepackage{url}
\usepackage{verbatim}
\usepackage{graphicx}
\usepackage{cite}
\usepackage{amsthm}
\usepackage{thmtools}
\usepackage{thm-restate}
\usepackage{float}
\newfloat{protocol}{htbp}{lop}
\floatname{protocol}{Protocol}
\declaretheorem[name=Theorem]{theorem}
\declaretheorem[name=Lemma,     sibling=theorem]{lemma}

\declaretheorem[name=Definition, sibling=theorem]{definition}
\declaretheorem[name=Assumption, sibling=theorem]{assumption}
\declaretheorem[name=Example,   sibling=theorem]{example}
\declaretheorem[name=Remark,    sibling=theorem]{remark}

\begin{document}

\title{MESHA: Mechanism-Enforced Sequential Halving\\ for Strategic Linear Bandits}

\author{Xin Li,~\IEEEmembership{Graduate Student Member,~IEEE}, Zixin Zhong,~\IEEEmembership{Member,~IEEE}
\thanks{Xin Li and Zixin Zhong are both with the Data Science and Analytics Thrust, Hong Kong University of Science and Technology (Guangzhou).}
}

%\markboth{Under Review of IEEE TRANSACTIONS ON INFORMATION THEORY, July 2026}
%{Journal of \LaTeX\ Class Files,~Vol.~1, No.~2, December~2023}%

%\IEEEpubid{Under Review~\copyright~2026 IEEE}

\maketitle

\begin{abstract}

We design and analyze \underline{M}echanism-\underline{E}nforced \underline{S}equential \underline{HA}lving (MESHA), an algorithm for Best Arm Identification (BAI) in strategic linear bandits. In this setting, each arm may strategically misreport its feature vector to maximize the probability of being identified as the best arm, when rewards are generated from the arms' true but unobservable features. The design of MESHA applies the na\"ive uniform sampling rule and an epoch-wise Grim Trigger Condition (GTC): the former reduces the impact of arms' strategic behaviours and the latter eliminates arms whose reported features severely deviate from the ground truth. Considering an arbitrary Nash Equilibrium, we prove that any arm would attempt to pass the GTC check to maximize its identified probability and derive an upper bound on the failure probability of MESHA within a fixed budget $T$. We also show that state-of-the-art linear BAI algorithms with $G$-optimal design would fail in such strategic environment, as the optimal design (OD)-based sampling rule based on strategically reported features may {\it starve} the optimal arm of any sampling budget. Finally, extensive numerical experiments indicate that MESHA 
 outperforms baselines that rely on OD-based sampling rules as well as the feature-agnostic baselines, corroborating the efficacy of MESHA.

\end{abstract}


\section{INTRODUCTION}

\IEEEPARstart{C}{onsider} a hiring platform seeking to hire the optimal candidate from a pool of $K$ applicants within a limited budget of $T$ interview rounds. Each candidate $i$ is inherently characterized by a true background profile, represented as a feature vector $x_{i} \in \mathbb{R}^d$. The platform evaluates these candidates based on a unified criterion modeled by an unknown latent vector $\theta^*$, and the true underlying score of candidate $i$ is represented by the inner product $\langle \theta^*, x_i \rangle$. At each 
round, the platform selects one candidate to interview, and obtains a noisy evaluation of that candidate.
In this setting, candidates act as self-interested agents whose primary objective is to maximize their probability of being selected as the final hire. To achieve this, they may strategically manipulate their submitted profiles, presenting a falsified feature vector $x_{t,i} \neq x_i$ to increase their chances of being selected for interview at any round $t$. However, once selected, their actual performance depends on their true profile $x_i$ rather than the reported feature $x_{t,i}$. 
Consequently, to identify the optimal candidate, the platform must judiciously select interviewees at each round, taking into account noisy evaluations and strategic profile manipulation.

To address such challenge for hiring platforms, we formalize the problem as a fixed-budget Best Arm Identification (BAI) task in strategic linear bandits. The learner's goal is to identify the optimal candidate within a limited budget of $T$ interview rounds. While state-of-the-art (SOTA) linear BAI algorithms have achieved minimax-optimal performance in classical linear bandits \cite{yang2022minimax}, they are fundamentally vulnerable to arms' strategic behavior. Specifically, these current SOTA methods rely on optimal design (OD)-based sampling rules such as $G$-optimal design and $\mathcal{XY}$-allocation~\cite{soare2014best}; however, we show in this work that such algorithm design allows self-interested arms to manipulate the sampling allocation by misreporting their 
features, which can lead to the failure of BAI.%learner's failure.

Recent work has begun to explore the intersection of reinforcement learning and mechanism design. For instance, \cite{zheng2022ai} proposed a two-level deep RL framework that jointly trains self-interested agents and a social planner to optimize tax policy in economic simulations, and \cite{kleine2024strategic,braverman2019multi} integrated mechanism design to to address strategic behavior in bandit problems with the objective of regret minimization.
Specially the OptGTM algorithm \cite{kleine2024strategic} introduces a Grim Trigger Condition (GTC) to discourage arms from misreporting their features in the contextual linear bandit setting. 
However, OptGTM's focus on regret minimization prevents it from efficiently identifying the best arm.
Despite these existing studies, the BAI task in strategic bandits remains under-explored.
To bridge this gap, we propose \underline{M}echanism-\underline{E}nforced \underline{S}equential \underline{HA}lving (MESHA), 
which integrates an epoch-wise GTC mechanism alongside a uniform sampling rule. Crucially, while prior work~\cite{kleine2024strategic,verma2025cobra} required stringent assumptions on feature reporting, by proving that the GTC inherently constrains arm behavior under a Nash Equilibrium, we establish performance guarantees for MESHA under substantially weaker and more practical assumptions.

\subsection{Literature Review}

Bandit literature broadly splits into two paradigms: Regret Minimization (RM) and Best-Arm Identification (BAI). First, in 
RM, a learner aims to maximize cumulative reward over $T$ rounds, with regret quantifying the gap between the learner’s cumulative reward and that of an oracle always pulling the optimal arm.
Second, the BAI (also called pure exploration) problem is studied under two settings: (i) in the fixed-budget setting, a learner aims to maximize the success BAI probability within a fixed budget $T$; (ii) in the fixed-confidence setting, a learner aims to minimize the sample complexity required for BAI with a fixed confidence $\delta$. The RM and BAI tasks necessitate fundamentally different algorithmic designs and theoretical analyses. More specifically, although both efficient RM and BAI algorithms need to carefully balance between exploitation and exploration,
exploitation is more critical in
achieving the optimal performance for RM,
whereas exploration is more crucial for BAI~\cite{lattimore2020bandit,zhong2021achieving,degenne2019bridging}.  
To contextualize our work, we first review existing literature on RM in both stochastic and linear bandit settings. Next, we discuss BAI works across both stochastic and linear settings under fixed-confidence and fixed-budget objectives. Finally, we examine existing work on strategic bandits.

\textbf{Regret Minimization.} The regret minimization problem in bandits has been extensively studied~\cite{auer2002finite,abbasi2011improved}. In the standard multi-armed bandit setting, \cite{lai1985asymptotically} established the first instance-dependent regret lower bound, and \cite{auer2002finite} proposed the UCB1 algorithm which 
achieves a regret nearly-matching this lower bound.\cite{thompson1933likelihood,agrawal2012analysis} studied the Thompson sampling algorithms, an alternative Bayesian approach that also 
achieves near-optimal regret while being computationally 
efficient in practice~\cite{chapelle2011empirical}.  Meanwhile, \cite{auer2002using} proposed the LINREL algorithm, 
one of the first algorithms for stochastic linear bandits. \cite{dani2008stochastic} 
established an $\Omega(d\sqrt{T})$ minimax lower bound on 
regret, and \cite{abbasi2011improved} proposed the OFUL algorithm, which achieves $\tilde{O}(d\sqrt{T})$ regret and is therefore minimax optimal up to logarithmic factors. %\zx{add more algorithms}

 \textbf{Best Arm Identification.}
In standard multi-armed bandits, the 
fixed-confidence BAI problem was studied by~\cite{even2002pac,even2006action}, who established Median Elimination and Action Elimination algorithms with $(\varepsilon,\delta)$-PAC guarantees via arm elimination. ~\cite{kaufmann2016complexity} later derived 
a universal instance-dependent lower bound on the sample complexity of any algorithm and proposed the 
asymptotically optimal Track-and-Stop algorithm, establishing a tight characterization of the sample complexity. In the fixed-budget setting, \cite{audibert2010best} proposed and analyzed the Successive Rejects algorithm and derived a lower bound on the failure probability of any algorithm with a Bernoulli instance, while \cite{karnin2013almost} proposed Sequential Halving (SH), achieving near-optimal sample complexity with an epoch-based design. 
Furthermore, \cite{zhao2023revisiting} refined the analysis of SH by deriving a superior upper bound under a sufficiently large budget $T$ and polynomially spaced arm gaps. For our comparison, we adopt the more general result of SH established in \cite{karnin2013almost}.
Moreover, \cite{shen2019universal,gabillon2012best} provided a unified study of both the fixed-confidence and fixed-budget settings.

In linear bandits, where arms' rewards are linear 
products of known feature vectors and a common, unknown latent vector, the fixed-confidence BAI problem was first studied by \cite{soare2014best}, who introduced the $\mathcal{XY}$-adaptive algorithm with an optimal design (OD)-based sampling strategy. Subsequently, the LinGapE algorithm proposed by \cite{xu2018fully} is with a better sample complexity guarantee and it is a fully adaptive algorithm.
%for fixed-confidence BAI tasks. 
Moreover, in the transductive setting, \cite{fiez2019sequential} designed the RAGE algorithm and derived the upper bound on its sample complexity. For the objective of 
fixed-budget BAI, \cite{yang2022minimax} designed the OD-LinBAI algorithm based on G-optimal design and proved that the upper bound of OD-LinBAI's failure probability matches the universal lower bound (also established in this work) up to logarithmic factors; hence, OD-LinBAI is minimax optimal and serves as the current SOTA. Note that $\mathcal{XY}$-adaptive and OD-LinBAI are both optimal design(OD)-based algorithms.

\textbf{Bandits with Strategic Arms.} 
The study of bandit problems with strategic arms was pioneered by~\cite{braverman2019multi}, who proposed a model where each arm receives a private stochastic reward $v_a$ upon being pulled, after which the pulled arm strategically reports a reward $\tilde{v}_a$ to the learner, retaining the residual $(v_a - \tilde{v}_a)$ for itself. Because each arm aims to maximize its cumulative retained reward over $T$ rounds, a fundamental conflict of interest arises: the learner seeks to maximize $\sum v_a$, whereas each arm is incentivized to under-report its reward. This severe information asymmetry can mislead the learner into pulling suboptimal arms, ultimately incurring linear regret.
Building on this, \cite{feng2020intrinsic} proved that classical algorithms such as UCB and $\epsilon$-greedy are naturally resistant to strategic reward inflation, as long as the manipulation budget of each arm grows sub-linearly with time. \cite{esmaeili2025robust} studied a similar model where arms are allowed to modify their rewards at a cost, and showed that UCB is not only robust to arms' strategic behavior but also can incentivize arms to perform rationally without introducing additional mechanism design. 

Meanwhile,\cite{esmaeili2023replication,shin2022multi} investigated an alternative manifestation of strategic environment, where arms can create copies of themselves to increase their chances of being selected; these works design mechanism-enforced algorithms — PI-ETC and H-UCB — that are robust to such replication strategies. 

More recently, \cite{kleine2024strategic, verma2025cobra} extended strategic learning to linear and contextual bandits by proposing the OptGTM and COBRA algorithms, respectively. To counter feature misreporting, OptGTM integrates a GTC mechanism into LinUCB. On the other hand, inspired by the Vickrey-Clarke-Groves (VCG) mechanism, COBRA contrasts an arm's individual optimistic reward estimate against a pessimistic estimate built from the history of all other arms. Crucially, the guarantees for both algorithms depend on strigent assumptions: (i) OptGTM relies on the condition that an arm's reported mean perfectly aligns with its true mean (Lemma E.1 of~\cite{kleine2024strategic}); (ii) COBRA demands that empirical UCB estimates consistently upper-bound true expected rewards under joint strategic behavior—a property the authors note fails when multiple agents manipulate data simultaneously (Assumption 1 of \cite{verma2025cobra}). Beyond these assumptions, both methods focus exclusively on RM but are without exploration for the BAI task. Finally, while their specific implementations differ, both OptGTM and MESHA build upon GTC principles, and hence we position OptGTM as the more relevant baseline for comparison.

To the best of our knowledge, no prior work has comprehensively explored the BAI task in the presence of strategic arms. To bridge this gap, we develop the MESHA algorithm, which couples a na\"ive sampling rule with a robust incentive mechanism to ensure a high success probability. Additionally, we expose how SOTA sampling rules for linear bandits break down under strategic environments, underscoring the necessity of our approach.

\subsection{Contributions}

In this work, we study the fixed-budget BAI problem in strategic linear bandits, where arms may strategically misreport their features to maximize their 
probability of being identified as the optimal arm. Our primary contributions are as follows:

\begin{itemize}
    \item \textbf{The MESHA Algorithm.} We propose 
    Mechanism-Enforced Sequential Halving (MESHA), a BAI algorithm designed for strategic linear bandits. MESHA employs a uniform sampling rule to prevent arms from manipulating the sampling allocation, and 
    incorporates an {\it epoch-wise} Grim Trigger Condition (GTC) to eliminate arms whose reported features severely deviate from the ground truth. Assuming an arbitrary Nash Equilibrium, we prove that every arm would attempt to pass the GTC check to maximize its identified probability, building on which we bound the failure probability of MESHA under a weaker and more practical assumption than that in existing works~\cite{kleine2024strategic,verma2025cobra}.

    \item \textbf{Theoretical Guarantees for MESHA.} We derive a lower bound on MESHA's failure probability of near-optimal arm identification under a fixed budget $T$, assuming an arbitrary Nash equilibrium. 
    MESHA's failure probability exhibits an exponential decay with respect to the budget $T$, demonstrating that the algorithm preserves the optimal convergence guarantees of non-strategic BAI baselines even within a strategic environment. Our analysis also reveals an inherent $\mathcal{O}(d^2 \log T)$ overhead, which arises from MESHA's need to disincentivize arms from deviating significantly from the ground truth.

    \item \textbf{Failure of SOTA Linear BAI Algorithms.}  
    We demonstrate that SOTA linear BAI algorithms~\cite{yang2022minimax} utilizing G-optimal design-based sampling rules are fundamentally susceptible to strategic manipulation. Absent an incentive mechanism design, strategic arms can directly mislead the learner via feature misreporting. Furthermore, even when equipped with a GTC check, these baselines are vulnerable to a novel \emph{starvation attack}. In this attack, suboptimal arms coordinate their reported features so that the optimal arm's feature vector is trapped within the cone spanned by the suboptimal features. This forces the G-optimal design to allocate zero samples to the optimal arm; as a result, the suboptimal arms can pass the GTC check while the optimal arm is completely starved of pulls and missed by the learner. We prove that this failure is structural — no statistical test evaluating only reported features can identify manipulations that distort the underlying feature geometry — thereby justifying both the uniform sampling and GTC framework embedded in MESHA.

    \item \textbf{Numerical Experiments.} 
    We conduct extensive experiments to evaluate the BAI performance of MESHA against several SOTA baselines. The empirical results demonstrate that MESHA consistently achieves successful BAI across varying budgets $T$, feature dimensions $d$, and arm 
    counts $K$, whereas OD-based algorithms collapse in strategic environments. 
    Furthermore, these experiments highlight the practical advantages of MESHA: first, MESHA's advantage over feature-agnostic algorithms widens significantly as $K$ increases; second, 
    its performance gap with OD-LinBAI~\cite{yang2022minimax} — even when the latter is evaluated in a non-strategic environment — 
    narrows as $d$ increases.
\end{itemize}

\section{Problem Formulation}
This section formalizes the BAI problem in strategic linear bandits. Section \ref{basicsetup} introduces the underlying learning dynamics. Section \ref{armsandequil} then defines the utility functions — which describe arms' strategic behaviors —and the corresponding Nash Equilibrium.
\subsection{Underlying Dynamics}%Basic Problem Setup}
\label{basicsetup}
For any $n \in \mathbb{N}$, we denote the set $\{1, \ldots , n\}$ as $[n]$. A random variable $X $ (or its distribution) is $\xi$-sub-Gaussian ($\xi$-SG) if
$\mathbb{E}\big[ \exp( \lambda (X -\mathbb E( X)) \big] \leq \exp( {\lambda^{2} \xi^{2}}/{2}).$ 
In particular, a random variable supported on $[a,b]$ is $\xi$-SG with $\xi=(b-a)/2$.
Let there be $K \in \mathbb{N}$ arms indexed as $[K]$. 
Each arm $i\in[K]$ is characterized by a true, unobservable feature vector $x_i \in \mathbb{R}^d$ and its quality is measured by $\mu_i := \langle \theta^*, x_i \rangle$, where $\theta^*$ is a unknown fixed latent vector. The unique optimal arm $i^*$ is defined by $\arg\max\limits_{i \in [K]} \mu_i$. The sub-optimality gap for any suboptimal arm $i$ is strictly positive, defined as $\Delta_i := \mu_{i^*} - \mu_i > 0$ for all $i \neq i^*$. 

Given a fixed budget of $T$ rounds, each arm acts as a self-interested agent aiming to maximize its probability of being identified as the best arm. To achieve this goal, each arm $i\in[K]$ may strategically report a manipulated feature vector $x_{t,i} \in \mathbb{R}^d$ at any round $t \in [T]$. 
At each round $t$, the learner observes the reported features $\{x_{t,1}, \dots, x_{t,K}\}$, selects an arm $i_t$ and then receives a noisy reward generated based on the selected arm's true feature vector:
\begin{equation}
r_{t,i_t} = \langle \theta^*, x_{i_t} \rangle + \eta_t,
\end{equation}
 where $\eta_t$ is an independent $\xi$-SG noise. 
Besides, each arm $i$ knows its true feature $x_i$ and the latent vector $\theta^*$. Moreover, all arms observe the learner's sequential sampling decisions $i_t$ and the corresponding reward realizations $r_{t,i_t}$ over time. We adopt a {\it full information} assumption that each arm knows 
the design of the learner's algorithm $\mathcal{A}$.
% , including its sampling rule and elimination criterion. 
This assumption is natural when the 
leaner's algorithm is publicly announced — such as a hiring 
platform that publishes its evaluation procedure, and 
represents the worst case for the learner since arms 
can tailor their misreporting strategies optimally 
against the known algorithm.

Without loss of generality, we assume that all true features $x_i$, all reported features $x_{t,i}$ and the latent vector $\theta^*$ are bounded in $\ell_2$-norm: $\|x_i\|_2 \le 1$, $\|x_{t,i}\|_2 \le 1$, and $\|\theta^*\|_2 \le 1$ for all $i \in [K]$ and $t \in [T]$; we also set $\xi=1$. 
\subsection{Arm Utility Function and Nash Equilibrium}
\label{armsandequil}
Let $\sigma_i$ denote arm $i$'s strategy, which maps its true feature $x_i$, latent parameter $\theta^*$, the algorithm $\mathcal{A}$ and the interaction history $\mathcal{H}_{t}=\left\{(x_{s,i_s},r_{s,i_s})\right\}_{s\leq t-1}$ to a reported feature $x_{t,i} \in \mathbb{R}^d$ at round $t$. We define the joint strategy profile of all arms as $\boldsymbol{\sigma} := (\sigma_1, \dots, \sigma_K)$, and the profile of all arms except $i$ as $\boldsymbol{\sigma}_{-i}$, i.e., $\boldsymbol{\sigma}_{-i}=(\sigma_1,\cdots,\sigma_{i-1},\sigma_{i+1},\cdots,\sigma_K)$. The {\it truthful reporting strategy,} where $x_{t,i} = x_i$ for all $t \in [T]$, is denoted by $\sigma_i^*$ for arm $i$. 
To maximize its probability of being selected as the best arm, each arm $i \in [K]$ strategically reports a manipulated feature vector $x_{t,i}$ according to its individual strategy $\sigma_i$. 
Let $i_{out,T}^{\mathcal{A},\boldsymbol{\sigma}}$ denote the arm identified as the best arm.
Our work studies the learner's BAI performance when the strategy profile $\boldsymbol{\sigma}$ reaches a Nash equilibrium. To formalize this strategic environment, we begin by defining the arm utility functions and the corresponding Nash equilibrium.

\begin{definition}{(Arm's Utility Function)} Given a learning algorithm $\mathcal{A}$, the utility function of each arm $i \in [K]$ under strategy profile $\boldsymbol{\sigma}$ is defined as the probability of being identified as the best arm, i.e., 
\begin{equation}
U_i(\mathcal{A}, \boldsymbol{\sigma}) := \mathbb{P}(i_{out,T}^{\mathcal{A},\boldsymbol{\sigma}} = i \mid \boldsymbol{\sigma}).
\end{equation}
\end{definition}

\begin{definition}{(Nash Equilibrium)}
    A strategy profile $\boldsymbol{\sigma}$ forms a Nash Equilibrium (NE) under algorithm $\mathcal{A}$ if, for any arm $i \in [K]$ and any alternative strategy $\sigma_i'$, the following holds:
\begin{equation}
U_i(\mathcal{A}, (\sigma_i, \boldsymbol{\sigma}_{-i})) \ge U_i(\mathcal{A}, (\sigma_i', \boldsymbol{\sigma}_{-i})).
\end{equation}
The set of all strategy profiles forming a Nash Equilibrium under $\mathcal{A}$ is denoted as $\text{NE}(\mathcal{A})$.
\end{definition} 
% Under the $(\zeta, T)$-PAC setting, a budget $T > 0$ is fixed. 
For fixed $\zeta\in \mathbb{R}^+$ and $T\in\mathbb{N}^+$, an algorithm $\mathcal{A}$ is said to be  $(\zeta, \delta_T)$-PAC {\emph (probably approximately correct)} if 
\begin{equation}
\mathbb{P} (\Delta_{i_{out,T}^{\mathcal{A},\boldsymbol{\sigma}}} \ge \zeta \mid 
\boldsymbol{\sigma} \in \text{NE}(\mathcal{A})) \leq \delta_T,
\end{equation}
where $\delta_T$ is a function of $T$.
Our goal is to design a 
$(\zeta, \delta_T)$-PAC algorithm $\mathcal{A}$ such that both $\zeta$ and $\delta_T$ are
as small as possible. This failure probability also reflects the hardness 
of BAI task under strategic manipulation.
We abbreviate $i_{out,T}^{\mathcal{A},\boldsymbol{\sigma}}$ as $i_{\mathrm{out}}$ when there is no ambiguity.
To further clarify the interaction among learner, arms and environment, we describe the dynamics in the following Protocol \ref{protocol}.
\begin{protocol}[h]
\label{protocol}
\caption{Interaction Protocol: BAI in Strategic Linear Bandits}
\begin{algorithmic}[1]
\STATE Learner picks algorithm $\mathcal{A}$ 
      given a fixed budget $T$ and arm set $[K]$.
\FOR{$t = 1, \ldots, T$}
    \STATE All arms report potentially gamed features $\{x_{t,1},\cdots,x_{t,K}\}$ based on $\boldsymbol{\sigma}$, learner selects arm $i_t \in [K]$ based on $\mathcal{H}_t$ and receives reward
           \begin{equation*}
               r_{t,i_t} := \langle \theta^*, x_{i_t} \rangle + \eta_{t}
           \end{equation*}
           where $\eta_t$ is zero-mean $\xi$-SG noise,
           $\theta^* \in \mathbb{R}^d$ is unknown latent vector, and the reward $r_{t,i_t}$
           is generated based on the true feature 
           $x_{i_t}$.%, instead of the reported $x_{t,i_t}$
\ENDFOR
\STATE Learner \textbf{outputs} the identified optimal arm $i_{\mathrm{out}}$ based on $\mathcal{H}_T$.
\end{algorithmic}
\end{protocol}
\section{THE MESHA ALGORITHM}

We now present the \underline{M}echanism-\underline{E}nforced \underline{S}equential \underline{HA}lving (MESHA) algorithm.
% and postpone detailed analysis to the appendices. 
% 
Designed for BAI in strategic linear bandits, MESHA divides the total budget $T$ into $R =\left\lceil \log_2 K \right\rceil$ epochs.
In each epoch $r\in[R]$, MESHA maintains an active set $\mathcal{A}_{r-1}$ and allocates the budget $n_r$ to sample each arm. Moreover, the design of MESHA rests on two core principles. First, to prevent being misled by potentially manipulated features, MESHA samples arms uniformly within each epoch. This ensures that the budget allocation is decoupled from the arms' reported features $x_{t,i}$. Second, to constrain arms' strategic behaviors, MESHA incorporates a mechanism called Grim Trigger Condition (GTC) at the end of each epoch. The GTC acts as a statistical consistency check between estimated rewards based on reported features and actual observed rewards generated with true features. Since all arms aim to maximize its utility, that is, the probability of being identified as the best arm,  they would attempt to increase the survival probability, and hence the threat of elimination via the GTC forces them to constrain their strategic deviations under an Nash Equilibrium formed by any $\boldsymbol{\sigma} \in \text{NE}(\text{MESHA})$. 
% To facilitate understanding, t
The pipeline of MESHA are provided in Algorithm \ref{alg:alg1} as well as  Figure \ref{fig:pipelineof} and also elaborated as below.

\begin{algorithm}[h]
\caption{\textbf{M}echanism-\textbf{E}nforced \textbf{S}equential \textbf{HA}lving (MESHA)} \label{alg:alg1}
\begin{algorithmic}[1]
\STATE {\bfseries Input:} total budget $T$, arm set $\mathcal{A} = [K]$, and target accuracy $\zeta$.% and parameter $\lambda$.
        \STATE {\bfseries Initialize:} $t_0=0$, $\mathcal{A}_0=\mathcal{A}$, $R=\lceil\log_2 K\rceil$, $S_0=\emptyset$, $\lambda=1$, and $\delta = \frac{\lceil\log_2 K\rceil^2}{T}
\exp\!\left(-\frac{T\zeta^2}{18Kd^2\log^2(1+T/
\lceil\log_2 K\rceil)}\right)$.
        \FOR{$r=1$ {\bfseries to} $R$}
        \STATE Pull each arm $i\in\mathcal{A}_{r-1}$ for \[
        n_r(i)=\left\lfloor\frac{T}{|\mathcal{A}_{r-1}|\lceil\log_2K\rceil}\right\rfloor
        \]
        times and update \[t_r=t_{r-1}+|\mathcal{A}_{r-1}|\cdot n_r(i),\quad S_r=\{t_{r-1},\dots,t_r-1\}.\]
        \STATE  Collect reported vectors $\{x_{t,1},\cdots,x_{t,K}\} \subset \mathbb{R}^d$ for all $t\in S_r$ and
        update statistics $V_{r,i}$ and $\hat{\theta}_{r,i}$ for arm $i\in\mathcal{A}_{r-1}$:\[
        V_{r,i}=\lambda I_d+\sum\limits_{t\in S_r}x_{t,i_t}x_{t,i_t}^\top\mathbf{1}(i_t=i), \quad 
        \hat{\theta}_{r,i}=V_{r,i}^{-1}\sum\limits_{t\in S_r}r_{t,i_t}x_{t,i_t}\mathbf{1}(i_t=i).
        \]
        \STATE Update the cumulative lower confidence bound of reward based on reported features for arm $i$:
        \[
        \text{RLCB}_{r,i}=\sum\limits_{t\in S_r}\Big(\langle\hat{\theta}_{r,i_t},x_{t,i_t}\rangle-\beta_{r,i}\|x_{t,i_t}\|_{V_{r,i_t}^{-1}}\Big)\cdot\mathbf{1}(i_t=i).
        \]
        \STATE Update the upper confidence bound of cumulative reward based on actual rewards for arm $i$:
        \[
            \text{AUCB}_{r,i}=\sum\limits_{t\in S_r}r_{t,i_t}\mathbf{1}(i_t=i)+\sqrt{2n_r(i)\log(4KR/\delta)}.
        \]
        \FOR{each arm $i\in\mathcal{A}_{r-1}$}
        \IF{$\text{RLCB}_{r,i}>\text{AUCB}_{r,i}$}
        \STATE Eliminate arm $i$ from set $\mathcal{A}_{r-1}$.
        \ELSE
        \STATE Update its estimated reward\[
        \hat{\mu}_{r,i}=\frac{\sum\limits_{t\in S_r}\langle\hat{\theta}_{r,i_t},x_{t,i_t}\rangle\cdot\mathbf{1}(i_t=i)}{n_r(i)}.
        \]
        \ENDIF
        \ENDFOR
        \STATE Update $\mathcal{A}_r$ as the set of $\min(|\mathcal{A}_{r-1}|,\lceil\frac{K}{2^r}\rceil)$ arms in $\mathcal{A}_{r-1}$ with largest $\hat{\mu}_{r,i}$.
        \ENDFOR
        \STATE{\bfseries Output} the single arm $i_{\text{out}}$ in set $\mathcal{A}_R$.
\end{algorithmic}
\label{alg1}
\end{algorithm}

\begin{figure}
    \centering
    \includegraphics[width=0.95\linewidth]{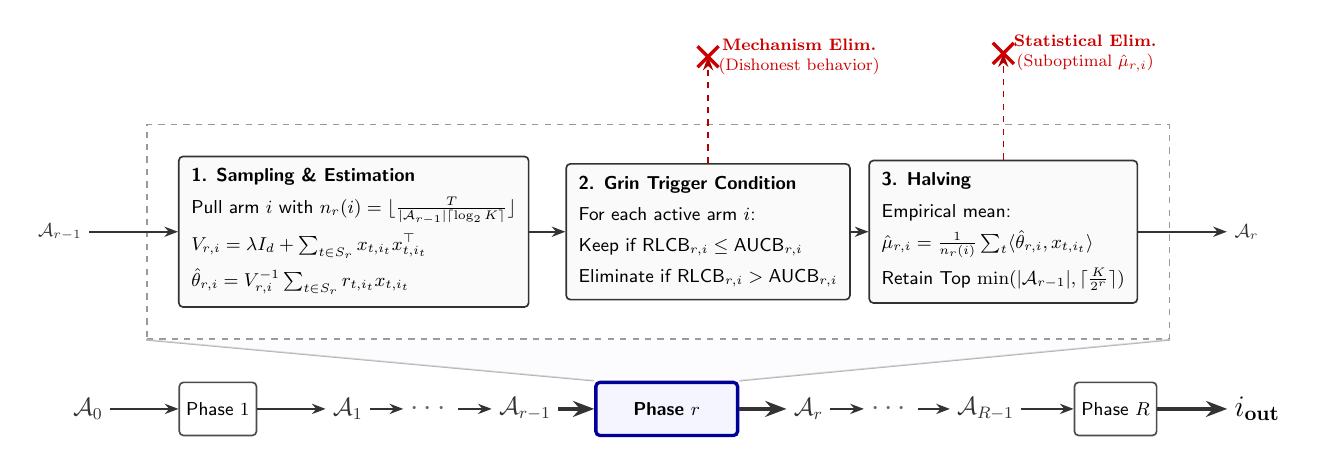}
    \caption{The pipeline of MESHA. In each epoch $r$, arms are first sampled uniformly and corresponding statistics are updated accordingly. The GTC check then eliminates arms exhibiting exceedingly dishonest behavior (Mechanism Elimination), and Sequential Halving removes the remaining suboptimal arms based on empirical means (Statistical Elimination). This process repeats over at most $R = \left\lceil \log_2 K \right\rceil$ epochs until a single arm $i_{\text{out}}$ is identified.}\label{fig:pipelineof}
\end{figure}
\subsection{Partition of Budget}

MESHA operates over a sequence of $R =\left \lceil \log_2 K \right\rceil$ epochs. In the beginning, the algorithm initializes the active arm set as $\mathcal{A}_0 = [K]$. In each epoch $r \in [R]$, the learner samples every active arm $i \in \mathcal{A}_{r-1}$ for an equal number of rounds, denoted by $n_r$.  The per-arm sampling budget for epoch $r$ is defined as:
\[
n_r = \left\lfloor \frac{T}{|\mathcal{A}_{r-1}| R} \right\rfloor
\]
where $|\mathcal{A}_{r-1}|$ is the number of active arms at the beginning of epoch $r$. As the algorithm progresses and the set $\mathcal{A}_{r-1}$ shrinks, the per-arm budget $n_r$ increases. This sampling rule allows MESHA to obtain increasingly accurate reward estimates for the remaining candidates in $\mathcal{A}_{r-1}$ during later epochs, and the partitioning design ensures that MESHA terminates and outputs one arm within the budget $T$. Such feasibility is formalized in the following lemma. 
\begin{lemma}
\label{lemma:feasible}
   $\sum\limits_{r=1}^R|\mathcal{A}_{r-1}|\cdot n_r(i)\mathbf{1}(i\in\mathcal{A}_{r-1})\leq T$.
\end{lemma}
\begin{proof}[Proof of Lemma \ref{lemma:feasible}]
    By definition of $\mathcal{A}_{r-1}$ and $n_r(i)$ in line 18 and line 7 in algorithm \ref{alg:alg1}, we have
\[
    \sum\limits_{r\in[R],i\in\mathcal{A}_{r-1}}|\mathcal{A}_{r-1}|\cdot n_r(i)=\sum\limits_{r=1}^R|\mathcal{A}_{r-1}|\cdot\lfloor\frac{T}{|\mathcal{A}_{r-1}|\lceil\log_2K\rceil}\rfloor
        \leq \sum\limits_{r=1}^R|\mathcal{A}_{r-1}|\cdot\frac{T}{|\mathcal{A}_{r-1}|\lceil\log_2K\rceil}
        =\sum\limits_{r=1}^R\frac{T}{\lceil\log_2K\rceil}
        =T.
\]
\end{proof}
\subsection{Estimation of Active Arms}

At the end of each epoch $r \in [R]$,
to isolate the potential effects of strategic arms, MESHA maintains a unique estimator $\hat{\theta}_{r,i}$ for each arm $i \in \mathcal{A}_{r-1}$, which depends merely on the interaction history with arm $i$ during epoch $r$:
\[ \hat{\theta}_{r,i} = V_{r,i}^{-1} \sum_{t \in S_r} r_{t,i_t} x_{t,i_t} \mathbf{1}(i_t = i), \]
where $V_{r,i} = \lambda I_d + \sum_{t \in S_r} x_{t,i_t} x_{t,i_t}^\top \mathbf{1}(i_t = i)$ and $S_r$ is the set of all rounds within epoch $r$. Based on these individual estimators, the estimated mean of arm $i$ is computed as follows:
\[ \hat{\mu}_{r,i} = \frac{\sum_{t \in S_r} \langle \hat{\theta}_{r,i}, x_{t,i_t} \rangle \cdot \mathbf{1}(i_t = i)}{n_r(i)}. \]
MESHA estimates rewards for each arm separately to decouple the potential strategic behaviors of different arms, which guarantees MESHA can detect the strategic level of each arm individually. This serves as the basis for the success of subsequent epoch-wise Grim Trigger Condition (GTC) mechanism.
\subsection{Epoch-wise Grim Trigger Condition}

The {\em epoch-wise} GTC in MESHA exploits a fundamental tension in strategic manipulation. 
An arm that inflates its reported features to appear more attractive 
cannot make its actual rewards look consistent with its authentic features. 
Specifically, if arm $i$ reports an exaggerated feature vector $x_{t,i}$, 
the learner's estimator $\hat{\theta}_{r,i}$ would predict high rewards 
for arm $i$. However, true rewards depend merely on the unobserved true feature $x_i$ 
and are unaffected by the manipulation. This creates a detectable gap: 
estimated cumulative rewards based on reported features might exceed observed cumulative rewards.

To spot arms' severe strategic behaviors, the {\em epoch-wise} GTC in MESHA detects the gap by comparing two confidence 
bounds. At the end of each epoch, the learner computes a lower confidence bound $\textbf{RLCB}_{r,i}$ on rewards that are predicted using reported features, and an upper confidence bound $\textbf{AUCB}_{r,i}$ on rewards that are actually observed for all arm $i\in\mathcal{A}_{r-1}$. If the former exceeds the latter for one arm, the arm's reported features seem to be inconsistent with its actual reward. In this case, it is suspicious that the arm behaves dishonestly and the learner hence eliminates this arm immediately and permanently. This GTC design is formalized in Definition \ref{def:gtc}. 
Crucially, when the round-wise GTC check in OptGTM~\cite{kleine2024strategic} depends on the full observation history, the GTC check in MESHA operates epoch-wise, evaluating only the history of the recent epoch. This decoupling mechanism also serves as a cornerstone for MESHA's performance guarantees.

\begin{definition}{(Grim Trigger Condition)}
\label{def:gtc}
    Let $\beta_{r,i}=\sqrt{d\log\big(\frac{1+n_r(i)}{\delta_r}\big)}+1$. At the end of each epoch $r$, if arm $i$ satisfies
    \[
   \text{RLCB}_{r,i} > \text{AUCB}_{r,i},
    \]
    MESHA will eliminate arm $i$ from the active set $\mathcal{A}_r$ and exclude it from all subsequent rounds, where
   \begin{align*}
       \text{RLCB}_{r,i}&:=\sum\limits_{t\in S_r,i_t=i}\langle\hat{\theta}_{r,i},x_{t,i}\rangle-\beta_{r,i}\|x_{t,i}\|_{V_{r,i}^{-1}},\\
       \text{AUCB}_{r,i}&:=\sum_{t\in S_r,i_t=i}r_{t,i_t}+\sqrt{2n_{r}(i)\log(2/\delta_r)}.
   \end{align*}
\end{definition}

The epoch-wise GTC imposes a survival constraint on the self-interested arms. Since an arm's utility immediately drops to zero upon elimination, rational arms are incentivized to bound their strategic deviations in order to pass the GTC check. 
Besides, we note that the learner's observed rewards should be close to the mean rewards. Formally speaking, with probability $1-\delta$, the event
\[\mathcal{E}^{\mathrm{noise}} := \left\{\left|\sum_{t \in S_r, i_t=i}
(r_{t,i_t} - \mu_i)\right| \leq 
\sqrt{\frac{n_r(i)\log(2KR/\delta)}{2}},\ 
\forall r\in[R],\ \forall i \in \mathcal{A}_{r-1}\right\}\]
holds. Conditioned on this event, we can characterize the equilibrium behavior of arms as follows.

\begin{restatable}{lemma}{lemGTCpass}
\label{lemma:GTCpass}
Fix any strategy profile $\boldsymbol{\sigma} \in \mathrm{NE}(\mathrm{MESHA})$. For any epoch $r \in [R]$ and any arm $i \in \mathcal{A}_{r-1}$, arm $i$ must pass the GTC check at the end of epoch $r$ conditioned on event $\mathcal{E}^{\mathrm{noise}}$.
\end{restatable}
\subsection{Failure probability bound of MESHA}
Given a fixed budget $T$, we now characterize the $(\zeta, T)$-PAC failure probability bound for MESHA assuming a corresponding Nash Equilibrium.

\begin{restatable}{theorem}{thmErrorBound}
\label{theorem:errorbound}
If the strategy profile $\boldsymbol{\sigma}$ forms a Nash Equilibrium under MESHA, then for any target accuracy $\zeta$ satisfying
\[\zeta \geq \frac{12d\sqrt{K\log_2K}\log(1+T)}{\sqrt{T}},\]
MESHA is with %satisfies
\[\mathbb{P}\big(\Delta_{i_{\mathrm{out}}}\geq\zeta\big) \leq \frac{\lceil\log_2K\rceil^2}{T}\exp\left(-\frac{T}{18\frac{K}{\zeta^2}d^2\log_2\!\left(1+\frac{T}{\lceil\log_2 K\rceil}\right)}\right).\]
\end{restatable}

Theorem~\ref{theorem:errorbound} shows the failure probability of MESHA decays exponentially with 
$T$, indicating that MESHA preserves the convergence behavior 
of non-strategic SOTA algorithms in the strategic environment. 
Moreover, the factors $K$ and $d^2\log T$ in the 
exponent denominator are due to the strategic arms as elaborated in the Remark \ref{rem:tightness}.

\begin{remark}[Origins of $K$ and $d^2\log T$ in the Exponent]
\label{rem:tightness}
First, the $K$ factor arises from the possibility of 
incorrect elimination of the optimal arm $i^*$ during epoch 
$r=1$. Specifically, under uniform sampling, each arm is 
pulled $n_1 = \lfloor T/(KR) \rfloor$ times during $r=1$. 
The optimal arm $i^*$ may be eliminated during epoch $r=1$ if some suboptimal arm 
$i \neq i^*$ achieves a higher estimated mean reward than $i^*$, 
i.e., $\hat{\mu}_{1,i} \geq \hat{\mu}_{1,i^*}$. Applying a 
union bound over the probalities of such events for all $K-1$ suboptimal arms, the failure 
probability due to elimination in epoch $r=1$ scales as $K \cdot 
\exp(-\zeta^2/\omega_{1,i}^2)$. As substituting $n_1 \propto T/K$ 
implies  $\omega_{1,i}^2 \propto K/T$, the failure probability 
becomes $\exp(-T\zeta^2/K)$, which explicitly yields the $K$ 
factor. Since epoch $r=1$ is with the smallest per-arm budget 
$n_1$ and $\omega_{r,i}$ is decreasing in $n_r$, the failure probability due to epoch $1$
results as the bottleneck of analyzing MESHA.

Second, the $d^2\log T$ factor arises from the epoch-wise GTC check step: the confidence radius $\beta_{r,i} = O(\sqrt{d\log n_r})$ roots from the application of elliptical 
potential lemma introduced by \cite{abbasi2011improved}, and the trace bound 
$\sqrt{\sum_{t: i_t=i}\|x_{t,i}\|^2_{V_{r,i}^{-1}}} \leq 
\sqrt{d}$ roots from the definition of $V_{r,i}$. When we bound the strategic deviation of arms under Nash Equilibrium, these two terms leads to the $(d^2\log T)$ factor. Both terms persist under any reporting strategy and are mutually independent.
In contrast, OD-LinBAI achieves a tighter $d$  
factor in non-strategic environment~\cite{yang2022minimax} because under G-optimal design, the optimal allocation 
$\pi$ satisfies $\max_i x_i^\top V(\pi)^{-1} x_i \leq d$ by 
definition, which directly bounds the prediction uncertainty 
without the additional $\sqrt{d}$ from the trace identity. More explanation can be found in Appendix~\ref{appendix:tightness}. 

However, as we will show in Section~\ref{section:lowerbound}, 
G-optimal design is fundamentally vulnerable to strategic 
manipulation and its theoretical guarantee would collapse.%does not hold anymore.
\end{remark}
\subsection{Performance Evaluation of MESHA}
To evaluate MESHA's BAI performance, we compare the failure probability bound of MESHA established in Theorem~\ref{theorem:errorbound} against the bounds of
two natural baselines, as summarized in 
Table~\ref{table:comparison}. 
We evaluate these methods based on their failure probabilities, despite nuances in their underlying environments and objectives. 
% While OD-LinBAI remains inherently vulnerable to strategic manipulation, 
%
First, SH is immune to geometric distortions and serves as a feature-agnostic baseline. 
Second, OD-LinBAI \cite{yang2022minimax} — the current SOTA algorithm for BAI in non-strategic linear bandits — although inherently vulnerable to strategic manipulation, serves as a feature-aware baseline.
% : it is the current SOTA 
% algorithm for BAI in non-strategic linear bandits. 
Our MESHA algorithm, by contrast, is carefully engineered to counter strategic behaviors. 
Furthermore, when SH and OD-LinBAI aim to identify the exact optimal arm $i^*$, MESHA accommodates arms' strategic behaviors to robustly identify a $\zeta$-optimal arm.
Nevertheless, 
comparing MESHA against the two baselines isolates the two key dimensions of the BAI task in strategic bandits: the benefit of exploiting the arm features and the cost of mitigating strategic behavior.

\begin{table}[ht]
\renewcommand{\arraystretch}{1.8}
\centering
\caption{Theoretical Comparison between MESHA and Baselines}
\label{table:comparison}
\begin{tabular}{l|c|c|c|c}
\hline
\textbf{Algorithm} & \textbf{Setting} & \textbf{Sampling Rule} & 
\textbf{Exponential Decay Rate} & \textbf{Strategic Robustness} \\ 
\hline
\textbf{MESHA} (Ours) & Strategic Linear & Uniform & 
$\exp \left( -\frac{T}{(K/\zeta^2) \cdot d^2 \log_2 T} \right)$ & 
Yes (under NE) \\ 
\hline
\textbf{Classical SH} \cite{karnin2013almost} & Stochastic & Uniform & 
$\exp \left( -\frac{T}{H_2 \log_2 K} \right)$ & 
Yes (feature-free) \\ 
\hline
\textbf{OD-LinBAI} \cite{yang2022minimax} & Linear & $G$-optimal & 
$\exp \left( -\frac{T}{d \log_2 d} \right)$ & 
No (truthful arms only) \\ 
\hline
\end{tabular}
\end{table}

\textbf{Comparison with Classical SH.} Classical SH achieves 
a failure probability of $\exp\big(-T/(H_2\log_2 K)\big)$ 
where $H_2 = \max_i\, i\Delta_{(i)}^{-2}$, and is naturally 
robust to strategic manipulation since it has no access to feature 
vectors. However, $H_2$ can be significantly larger 
than $(K/\zeta^2)d^2$ when there exists a near-optimal 
arm with a significantly small suboptimality gap in the instance. Specifically, consider 
an instance with $K=16$, $d=4$, $\zeta=0.1$, where arm $1$ is 
the optimal arm, arm $2$ is with suboptimality gap 
$\Delta_2 = \zeta/(2d\sqrt{K}) \approx 0.003$, and the remaining 
arms $i \in \{3,\ldots,16\}$ are with suboptimality gap 
$\Delta_i = 2\zeta = 0.2$. In this case, 
$H_2 = \max(1 \cdot \Delta_2^{-2}, 2\cdot(2\zeta)^{-2}) 
= \Delta_2^{-2} = 102400 \gg 25600 = (K/\zeta^2)d^2$, indicating that the failure probability of SH can be far higher than that of MESHA. We note that since $\Delta_2 < \zeta$, this probability gap may partially result from that identification of arm $2$ is regarded as success for MESHA but not for SH.
Nevertheless, in scenarios where the $\zeta$-optimal arms are acceptable, MESHA achieves a higher success probability than SH by exploiting the strategic arm features.

\textbf{Comparison with OD-LinBAI.} OD-LinBAI achieves a 
smaller failure probability $\exp\big(-T/(d\log_2 d)\big)$ because 
under G-optimal design, the optimal allocation $\pi$ satisfies 
$\max_i x_i^\top V(\pi)^{-1}x_i \leq d$ by definition, 
which naturally bounds the prediction uncertainty with a single $d$ 
factor rather than $d^2$. However, this advantage relies 
heavily on optimal design-based sampling and is fundamentally 
vulnerable to the strategic setting: as to be elaborated in 
Section~\ref{section:lowerbound}, suboptimal arms can coordinate 
their reported features so that the optimal arm's feature vector is trapped within the cone spanned by the suboptimal features, causing G-optimal design to allocate zero 
pulls to the optimal arm while suboptimal arms can pass the GTC check and the optimal arm would be missed by the learner; this justifies the uniform sampling rule embedded in MESHA.
\section{Further Discussion on Algorithms for Linear Bandits}
\label{section:lowerbound}

Having introduced MESHA, a natural question arises: why not simply equip the current SOTA linear BAI algorithm \cite{yang2022minimax} with a GTC mechanism or apply OptGTM~\cite{kleine2024strategic} directly for BAI instead? We now answer this question by showing that such approaches would fail from the following fundamental perspectives. First, without any mechanism design, strategic 
arms can directly corrupt the learner's estimator by misreporting their features. Second, even when equipped with a GTC check, strategic arms can still evade it through a \textit{starvation attack} --- 
a coordinated manipulation where suboptimal arms misreport their features to ensure that the optimal arm receives zero pulls under $G$-optimal design-based sampling rules, while remaining consistent with their own 
rewards under the GTC check. Furthermore, we also see that the success of round-wise GTC check used in OptGTM \cite{kleine2024strategic} relies on a strong assumption which sidesteps the equilibrium analysis. In contrast, our epoch-wise GTC check is designed to actively constraint the arms' behavior, and our analysis only build on a strictly weaker and more reliable assumption.
These altogether justify both the uniform sampling rule and the GTC mechanism integrated in MESHA.

\subsection{Vulnerability of OD-LinBAI without Mechanism Design}

The current SOTA linear BAI algorithm OD-LinBAI \cite{yang2022minimax} 
relies on $G$-optimal design-based sampling rule to allocate pulls based on the 
reported feature vectors of arms. This sampling strategy 
assumes the reported features reflect the true underlying 
geometry of the arms. When arms are strategic, however, this 
assumption may no longer hold: arms can misreport their features to 
manipulate the sampling process and corrupt the learner's 
estimator. We illustrate this through the following example.

\begin{example}
\label{example:example1}
Consider a $2$-armed linear bandit. The true 
action space is defined as $\mathcal{X}=\{x_1, x_2\}$, where 
$x_1=(0.9, 0)$ and $x_2=(0.1, 0)$. The unknown environment 
parameter is $\theta^*=(1,0)$. At any round $t$, pulling arm 
$i_t$ yields a noisy reward $r_{t,i_t} = \langle\theta^*, 
x_{i_t}\rangle + \eta_t$, where $\eta_t \sim 
\text{Unif}(-0.05, 0.05)$. Consequently, the true expected 
rewards are $\mu_1 = 0.9$ (the optimal arm) and $\mu_2 = 0.1$. 
Suppose both arms strategically report identical features to  a \textbf{G-optimal design-based learner}: $x_{t,1}=x_{t,2}=(1,0)$ for all $t\in[T]$. Then the 
learner is unable to distinguish the two arms, and the failure 
probability approaches $1/2$.
\end{example}

\begin{proof}
Observing the reported set $\tilde{\mathcal{X}}=\{\tilde{x}_1=
(1,0), \tilde{x}_2=(1,0)\}$, any sampling rule based on 
$G$-optimal design will allocate pulls uniformly between the two 
arms. The learner computes the global least-squares estimator 
$\hat{\theta}_t$ as:
\[
\hat{\theta}_t=\Big(\sum_{s=1}^t x_{s,i_s}x_{s,i_s}^\top
\Big)^{-1}\sum_{s=1}^t x_{s,i_s} r_{s,i_s}.
\]
Since the reported features have a non-zero component only in the 
first dimension, the first element of the estimator simplifies to:
\[
\hat{\theta}_t^{(1)}=\frac{1}{t}\sum_{s=1}^t r_{s,i_s}.
\]
By the Law of Large Numbers, as $t \to \infty$, 
$\hat{\theta}_t^{(1)} \to \mathbb{E}[r] =0.9\times 0.5+0.1\times 0.5=0.5$. Thus, the 
empirical mean of both arms approaches $0.5$ when the algorithm proceeds, and the learner 
is unable to distinguish the optimal arm, resulting in 
an failure probability approaching $1/2$.
\end{proof}

In Example \ref{example:example1}, the manipulation succeeds 
because the two arms report identical features, causing \textbf{the 
optimal design-based algorithm}'s estimator to blend the rewards of both arms. The resulting estimator converges to a weighted average of 
both arms' true rewards, rather than the true reward of 
each individual arm, making identification impossible. This shows that without any 
mechanism design, even a trivial misreport is 
sufficient to cause identification failure.
\subsection{Vulnerability of OD-LinBAI with Mechanism Design} \label{sec:odlinbai_gtc} %Under the Grim Trigger Condition  

The failure of OD-LinBAI in Example \ref{example:example1} suggests that some 
form of consistency check is needed to detect manipulation. A 
natural candidate is OD-LinBAI-GTC, which combines OD-LinBAI with the same epoch-wise GTC used in MESHA. The pseudocode is shown in Algorithm \ref{alg:alg2}. 
However, we find that strategic arms can still pass this check 
through a \textit{starvation attack}, by exploiting a structural 
property of $G$-optimal design: if the optimal arm's reported 
feature falls within the cone spanned by the suboptimal arms' 
reported features, then $G$-optimal design allocates zero pulls 
to the optimal arm. We formalize this phenomenon in the 
following theorem and postpone its proof to Appendix \ref{appendix:C}.

\begin{restatable}{theorem}{thmMisthm}
\label{theorem:misthm}
Let $x_1,\cdots,x_K\in\mathbb{R}^d$. Suppose there exists an index $j$ and non-negative coefficients $\{\lambda_i\}_{i\in[K]\setminus\{j\}}$ such that $\sum_{i\neq j}\lambda_i\leq 1$ and $x_j=\sum_{i\neq j}\lambda_ix_i$. Then, in the $G$-optimal design problem:
\[\min_{\omega\in\Delta^K}\max_{i=1,\dots,K} x_i^\top V(\omega)^{-1}x_i, \quad V(\omega)=\sum_{i=1}^K\omega_i x_i x_i^\top,\]
there exists an optimal allocation $\omega^*$ such that $\omega_j^*=0$.
\end{restatable}

Theorem \ref{theorem:misthm} shows that suboptimal arms can 
coordinate their reported features to starve the optimal arm of 
any sampling budget, while passing the GTC by remaining 
consistent with their own rewards. We illustrate this attack 
through the following example.

\begin{example}
\label{example:example2}
Consider a $3$-armed bandit with budget $T$. The true action space consists of 
$x_1=(0.5, 0, 0)$ and $x_2=x_3=(0.45, 0, 0)$. Given the 
true parameter $\theta^*=(1,0,0)$, the expected rewards are 
$\mu_1=0.5$ (the optimal arm) and $\mu_2=\mu_3=0.45$. The noisy 
reward follows $r_{t,i_t} = \mu_{i_t} + \eta_t$, where $\eta_t 
\sim \mathcal{N}(0, \sigma^2)$. When arms 2 and 3 strategically coordinate 
their reported features to starve arm 1 of any pulls. Then arm 1 can be deterministically eliminated, 
leading to a identification failure of OD-LinBAI-GTC.
\end{example}

\begin{proof}
Suppose the arms adopt the following strategies:
\begin{itemize}
    \item \textbf{Arm 1}: Reports $x_{t,1}=(0, 1/2, 1/6)$.
    \item \textbf{Arm 2}: Reports $x_{t,2}=(0, 0.9, 0)$ with 
    pseudo-parameter $\theta^*_2=(0, 0.5, 0.5)$ satisfying 
    $\langle\theta^*_2, x_{t,2}\rangle = 0.45 = \mu_2$.
    \item \textbf{Arm 3}: Reports $x_{t,3}=(0, 0, 0.9)$ with 
    pseudo-parameter $\theta^*_3=(0, 0.5, 0.5)$ satisfying 
    $\langle\theta^*_3, x_{t,3}\rangle = 0.45 = \mu_3$.
\end{itemize}
Since $x_{t,1} = \frac{5}{9}x_{t,2} + \frac{5}{27}x_{t,3}$ 
with $\frac{5}{9} + \frac{5}{27} = \frac{20}{27} < 1$, 
Theorem~\ref{theorem:misthm} guarantees there exists an optimal G-design 
allocation $\omega^*$ with $\omega^*_1 = 0$, so OD-LinBAI 
exclusively samples Arms 2 and 3 throughout. Since the true 
rewards of Arms 2 and 3 are generated from their true features 
as $r_{t,i} = \langle\theta^*, x_i\rangle + \eta_t = 
0.45 + \eta_t$, and their pseudo-parameters satisfy 
$\langle\theta^*_i, x_{t,i}\rangle = 0.45 = \mu_i$, the 
GTC-predicted rewards are consistent with the observed rewards, 
so Arms 2 and 3 can pass the GTC check. Since only Arms 2 and 3 
are sampled and their reported features have zero first 
coordinate, the global OLS estimator satisfies 
$\hat{\theta}_t \to \theta^*_2 (= \theta^*_3)$ as $t \to 
\infty$. When $t$ grows, the estimate of Arm 1 approaches 
$\langle\theta^*_2, x_{t,1}\rangle = 0.4 < 0.45$, so Arm 1 
is ranked below Arms 2 and 3 and deterministically eliminated, 
causing identification failure.
\end{proof}
In this example, we construct a starvation attack where arms' rewards depend only on the first feature coordinate. Under strategic reporting, every arm sets the first coordinate of its reported feature to zero, hiding the reward-relevant direction entirely. This places optimal design-based algorithms under maximal strategic pressure while leaving feature-agnostic algorithms unaffected.
The {\it{starvation attack}} succeeds because suboptimal arms coordinate 
their reported features to manipulate the geometric boundary of 
the feature space, ensuring that the optimal arm is never sampled 
while their own rewards remain consistent considering the GTC. This 
implies that the failure of $G$-optimal design in strategic settings 
is structural: no consistency check applied at the reward level 
can detect a manipulation that operates at the feature geometry 
level. In contrast, MESHA avoids this style of failures by 
applying uniform sampling, which ensures that every active arm 
receives pulls regardless of the reported feature geometry.
\subsection{Comparison against OptGTM}

The OptGTM algorithm, recently designed by \cite{kleine2024strategic}, also incorporates a GTC check 
% into a linear bandit algorithm 
to handle strategic behavior in linear bandits but for RM. However, 
the utilization of GTC in \cite{kleine2024strategic} reveals a fundamental 
limitation. Specifically, their analysis relies on Lemma E.1 of \cite{kleine2024strategic}, 
which assumes that for any arm $i$ and any round $t$ where 
$i_t = i$, we have
\begin{align}
\langle\theta^*, x^*_{t,i}\rangle = \langle\theta^*_i, 
x_{t,i}\rangle, \label{eq:assum_optGTM}
\end{align}
arms' 
strategic manipulation has no effect on the reward evaluation. 
Under this assumption, every arm automatically passes the GTC check 
regardless of its actual behavior, rendering the GTC check passive rather than an active mechanism. We think this strong assumption 
might be introduced because finding a Nash Equilibrium directly proved 
difficult, and it effectively sidesteps the equilibrium analysis 
by pre-imposing honest behavior on the arms. 
In fact, the {\it{starvation attack}} constructed in Example \ref{example:example2} provides an explicit instance where the assumption in Lemma E.1 of \cite{kleine2024strategic} is violated: arms 2 and 3 report features inconsistent with their true features while using pseudo-parameters $\theta_2^*$ and $\theta_3^*$ to maintain reward consistency, a strategy that OptGTM's assumption would rule out by construction.

Rather than assuming 
that arms' reported mean values are equivalent to true mean values as in \eqref{eq:assum_optGTM}, MESHA applies the GTC check to actively constrain 
arms' strategic behavior to form a NE under MESHA where  the reported rewards remain close to the true rewards, but they do not necessarily coincide. As shown in Lemma \ref{lemma:GTCpass}, passing the GTC check
is a necessary condition for any arm to survive elimination of MESHA under 
Nash Equilibrium: a rational arm that expects to be selected in 
future rounds would attempt to satisfy the GTC, since failing it 
leads to permanent elimination and zero utility. In brief, the epoch-wise GTC of MESHA naturally constrains the arms' strategic behavior under a milder condition. 
% In other words, i
Instead of the strong assumption made in Lemma E.1 of OptGTM\cite{kleine2024strategic}, our analysis merely requires the following 
relaxed condition.

\begin{assumption}
\label{assumption:relaxed}
For any arm $i$ pulled at round $t$, the strategic deviation 
satisfies:
\[
\langle\theta^*_i, x_{t,i}\rangle - \langle\theta^*, 
x_i\rangle \leq \varepsilon.
\]
\end{assumption}

This is a weaker condition than Lemma E.1 for OptGTM\cite{kleine2024strategic}: 
it allows arms to misreport their features, as long as the 
resulting deviation in the reward evaluation remains bounded by 
$\varepsilon$, rather than being exactly zero. Under this 
assumption, we establish the following guarantee on the behavior of any arm that survives the GTC check of MESHA and postpone the proof to Appendix \ref{appendix:C}.

\begin{restatable}{theorem}{thmManipulation}
\label{theorem:manipulationlevel}
Conditioned on Assumption~\ref{assumption:relaxed}, any arm $i$ passing the GTC check of MESHA at the end of epoch $r$ satisfies:
\[\bar{\varepsilon}_{r,i} \leq \left(\frac{\beta_{r,i}}{\sqrt{n_r(i)}} + \varepsilon\right)\sqrt{2d\log(1 + n_r(i))} + 3\sqrt{\frac{\log(2/\delta_r)}{n_r(i)}},\]
where $\bar{\varepsilon}_{r,i} = \frac{1}{n_r(i)}\sum_{t \in S_r, i_t=i}\left(\langle\theta^*_i, x_{t,i}\rangle - \mu_i\right)$ represents the average per-round strategic deviation of arm $i$ during epoch $r$.
\end{restatable}

Theorem \ref{theorem:manipulationlevel}  implies that the epoch-wise GTC check in MESHA actively constrains 
the strategic behavior of any surviving arm: the average deviation $\bar{\varepsilon}_{r,i}$ is bounded in terms of the confidence radius $\beta_{r,i}$ and the number of pulls 
$n_r(i)$. As $n_r(i)$ grows, this bound tightens, indicating that arms face increasingly stringent constraints on their strategic behavior as the MESHA algorithm progresses. Furthermore, when $\varepsilon = 0$, 
Assumption \ref{assumption:relaxed} reduces to the assumption for OptGTM (see Lemma E.1 in \cite{kleine2024strategic}), indicating that the assumption in \cite{kleine2024strategic} is a special case of 
our framework.
\section{Numerical Experiments}
\label{section:exps}

We evaluate MESHA against five baselines: Sequential Halving 
\cite{karnin2013almost}, Successive Rejects\cite{audibert2010best}, OD-LinBAI \cite{yang2022minimax}, OD-LinBAI-GTC (introduced in Section~\ref{sec:odlinbai_gtc})
and OptGTM \cite{kleine2024strategic}.
We construct several strategic instances motivated by the {\it{starvation attack}} described in Section~\ref{section:lowerbound}. 
As baseline algorithms are generally designed for optimal arm identification, we set $\zeta=\Delta_{\min}/2$ for MESHA and compare the failure probabilities of optimal arm identification of all algorithms.
All results are averaged over 5000 independent trials and error bars indicate $95\%$ wald confidence intervals.
More details are provided in Appendix~\ref{appendix:G}.

\subsection{Overall Comparison}
\begin{figure}[htbp]
  \centering
  \subfloat{%
    \includegraphics[width=0.45\linewidth, trim={0cm 1cm 0cm 0cm},clip]{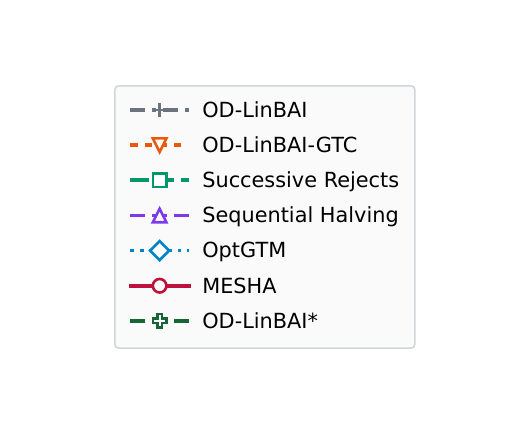} %[trim={left bottom right top},clip]
  }%
  \hfill
  \subfloat[Impact of Budget $T$]{%Overall Comparison(vary $T$)]{%
    \includegraphics[width=0.45\linewidth]{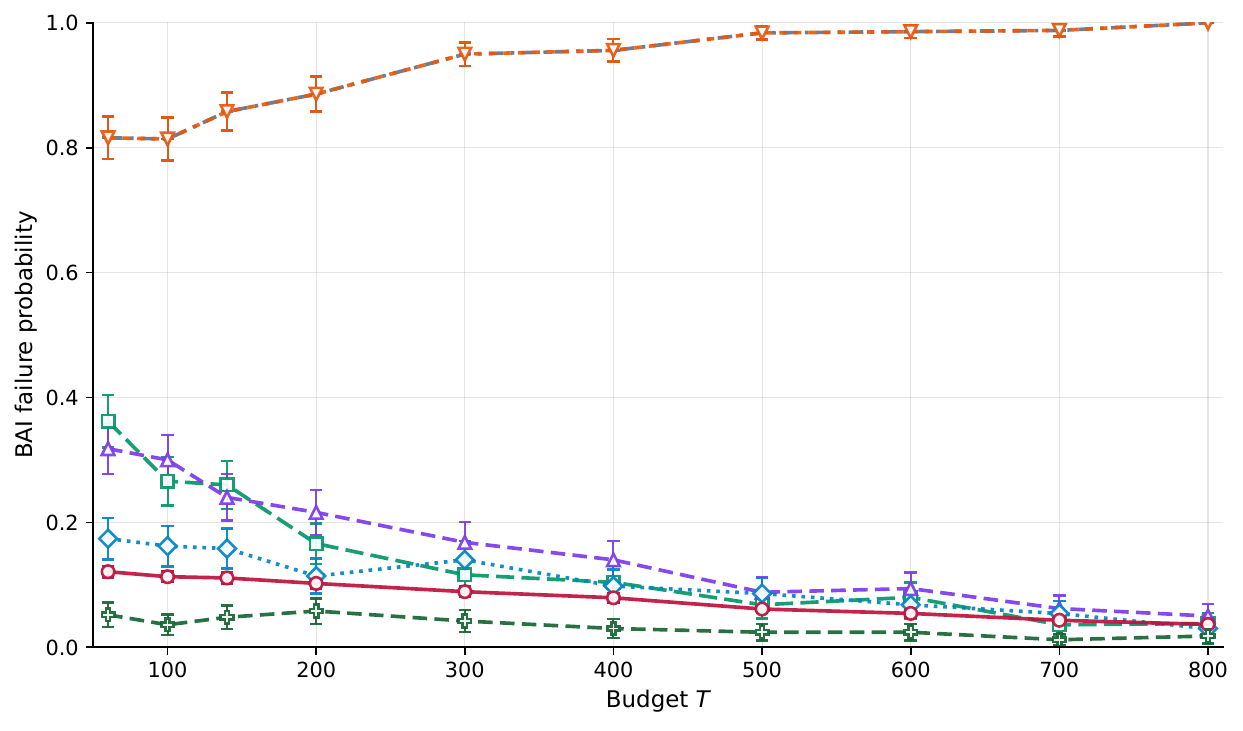}
    \label{fig:exp_a}
  }%
  \\[8pt] 
  \subfloat[Impact of Feature Dimension $d$]{%
    \includegraphics[width=0.45\linewidth]{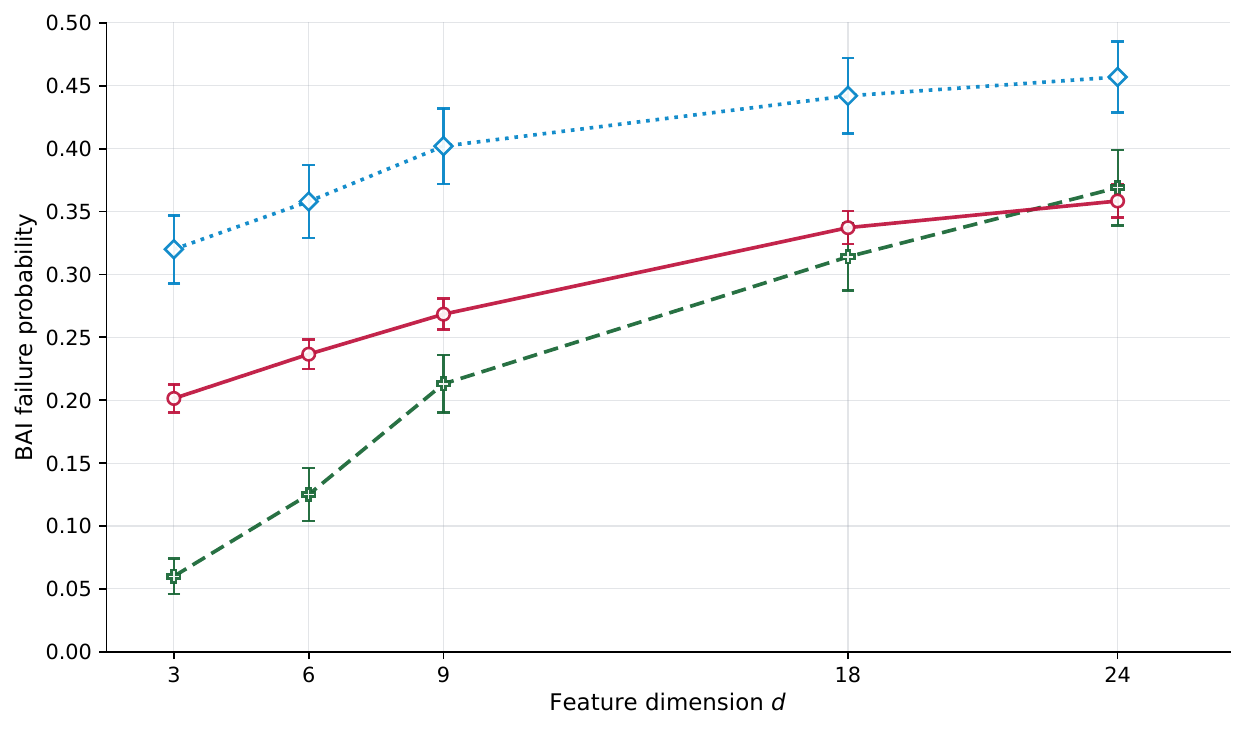}
    \label{fig:exp_b}
  }%
  \hfill
  \subfloat[Impact of Arm Count $K$]{%
    \includegraphics[width=0.45\linewidth]{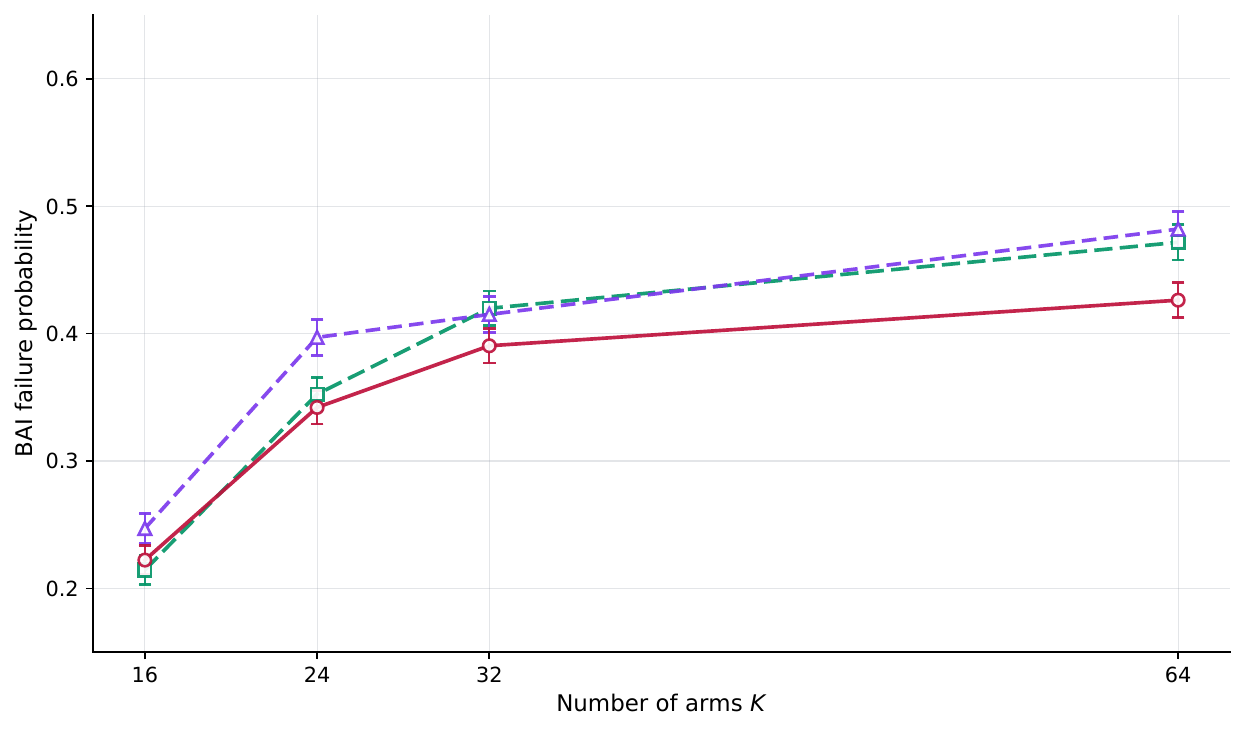}
    \label{fig:exp_c}
  }%

  \caption{Empirical BAI failure probability of all algorithms under the strategic instance except for that OD-LinBAI* represents OD-LinBAI itself when facing truthful arms. We compare MESHA against $5$ baseline algorithms.
  \textbf{(b)} With fixed $K=8$, $d=3$ and varying $T$: OD-LinBAI 
  fail as $T$ grows while MESHA and OptGTM 
  achieve consistent failure decay. Sequential Halving and GSE 
  decrease slowly due to lack of feature information.
  \textbf{(c)} With fixed $K=64$, $T=1{,}500$ and varying $d$: MESHA 
  consistently outperforms OptGTM; OD-LinBAI* starts with 
  lower failure probability at small $d$ but performs similarly to MESHA as $d$ 
  grows.
  \textbf{(d)} With fixed $d=6$, $T=1{,}500$ and varying $K$: MESHA 
  maintains the lowest failure probability across all arm counts, with the 
  margin over GSE and SH widening as $K$ increases.}
  \label{fig:fourfigs}
\end{figure}

In Figure \ref{fig:exp_a}, we vary $T \in \{60, 100, \dots, 800\}$ with fixed $K = 8$, 
$d = 3$. MESHA's failure probability decays from $0.188$ when $T = 60$ 
to $0.056$ when $T = 800$, which exponentially decay on $T$, corroborating Theorem \ref{theorem:errorbound}. 
OptGTM follows a similar decay trend and occasionally performs comparably to MESHA when the budget $T$ is large. However, OptGTM is designed for RM in strategic linear contextual bandits and is with no theoretical guarantee for the BAI tasks.
Both SH and SR are unaffected by strategic feature reporting since they are feature-ignostic. However, their failure probabilities decrease slowly, as they cannot utilize arm features to facilitate arm identification. OD-LinBAI and OD-LinBAI-GTC both fail completely, with failure probabilities approaching $1$ as $T$ grows, corroborating the {\it{starvation attack}} study in Section~\ref{section:lowerbound}. Notably, integrating the GTC mechanism does not help OD-LinBAI, implying that 
the failure is structural. 
In contrast, OD-LinBAI*, which observes the true feature vectors, achieves near-zero failure probability throughout, demonstrating that the failure of OD-LinBAI is entirely due to strategic 
feature misreporting rather than any algorithmic deficiency of OD-LinBAI itself.

\subsection{Impact of Feature Dimension and Arm Count}
In Figure \ref{fig:exp_b} and Figure \ref{fig:exp_c}, we fix $T = 1{,}500$ and individually vary $d$ and $K$ to 
evaluate the impact of feature dimension and arm count on 
MESHA's performance. While Theorem~\ref{theorem:errorbound} 
shows that the failure probability grows with both $d^2$ 
and $K$, numerical results reveal that 
MESHA's empirical performance is consistently better than 
this worst-case theoretical prediction, and its advantage 
over the baselines widens as either $d$ or $K$ grows.

\paragraph{Varying $d$}
With $K = 64$ fixed, we vary 
$d \in \{3, 6, 9, 18, 24\}$ to evaluate the impact of 
feature dimension on MESHA's performance.
Although Theorem~\ref{theorem:errorbound} shows an 
$\mathcal{O}(d^2\log T)$ penalty inside the exponent, 
MESHA's empirical failure probability increases only 
moderately with $d$, suggesting that the theoretical bound 
is conservative and MESHA performs better in practice than 
what the worst-case analysis implies.
MESHA consistently outperforms OptGTM by roughly $11$ 
percentage points across all values of $d$. Since OptGTM 
relies on feature-dependent sampling, it remains more 
susceptible to strategic manipulation than MESHA, whose 
uniform sampling rule is robust by design.
Comparing with OD-LinBAI*, which serves as an oracle 
baseline with access to true feature vectors, reveals a 
further advantage of MESHA. When $d = 3$, OD-LinBAI* 
achieves a much lower failure probability ($\approx 6.5\%$) 
due to the sample efficiency of $G$-optimal design in low 
dimension space. However, its failure probability rises steeply as $d$ grows 
and is similar to that of MESHA near $d = 18, 24$. 
This indicates that, despite the $d^2$ term in 
Theorem~\ref{theorem:errorbound}, MESHA is a practically 
competitive algorithm at large $d$: it achieves success probability comparable to the oracle baseline OD-LinBAI*.

\paragraph{Varying $K$}
With $d = 6$ fixed, we vary $K \in \{16, 24, 32, 64\}$ 
to evaluate the impact of arm count on MESHA's performance.
Although Theorem~\ref{theorem:errorbound} shows that 
the failure probability grows linearly with $K$ in the 
error exponent, MESHA's empirical failure probability 
increases at a rate noticeably slower than linear rate 
suggesting an $o(K)$ growth in practice. MESHA attains 
the lowest failure probability than non-linear baselines SR and SH, 
and its margin over the baselines widens as $K$ grows, 
further indicating that MASHE performs better in practice.
Both Sequential Halving and Successive Rejects are unaffected by 
strategic feature reporting since they are feature-ignosmic. However, their failure probabilities 
increase more rapidly than MESHA's as $K$ grows, since 
they rely solely on the reward signal but no feature information for arm identification.
MESHA retains a compounding advantage as $K$ increases. 
When the arm set is large and reward gaps 
are small, the additional discriminative power from features becomes increasingly valuable, 
and MESHA is the only algorithm that exploits this 
structure while remaining robust to strategic reporting.
\section{Proof Sketch of Theorem~\ref{theorem:errorbound}}
\label{section:sketch}

The key technical challenge in analyzing the failure 
probability of MESHA is that arms' strategic behaviors 
corrupt the feature-based estimators $\hat{\theta}_{r,i}$, 
rendering the standard concentration arguments commonly 
applied in truthful linear bandits~\cite{abbasi2011improved} 
inapplicable. To overcome this challenge,  we begin our analysis with defining a \textit{``good event''} based on the learner's observed rewards.

\begin{restatable}{lemma}{lemConcentration}
\label{lemma:concentration}
Let $\delta = \frac{\lceil\log_2 K\rceil^2}{T}
\exp\!\left(-\frac{T\zeta^2}{18Kd^2\log^2(1+T/
\lceil\log_2 K\rceil)}\right)$ and define a \it{ ``good event''} based on the the learner's observed rewards as:
\[\mathcal{E}^{\mathrm{noise}} := \left\{\left|\sum_{t \in S_r, i_t=i}
(r_{t,i_t} - \mu_i)\right| \leq 
\sqrt{\frac{n_r(i)\log(2KR/\delta)}{2}},\ 
\forall r\in[R],\ \forall i \in \mathcal{A}_{r-1}\right\}.\] Then  $\mathcal{E}^{\mathrm{noise}}$ holds with probability at least
$1 - \delta$.
\end{restatable}
Conditioned on $\mathcal{E}^{\mathrm{noise}}$, Lemma \ref{lemma:GTCpass} implies that all active arms aim to pass GTC check under Nash Equilibrium, and hence we can derive the following strategic deviation bound for active arms.% under Nash Equilibrium.
\begin{restatable}{lemma}{lemGTCbound}
\label{lemma:gtc_bound}
Let $\delta = \frac{\lceil\log_2 K\rceil^2}{T}
\exp\!\left(-\frac{T\zeta^2}{18Kd^2\log^2(1+T/
\lceil\log_2 K\rceil)}\right)$ and assume $\mathcal{E}^{\mathrm{noise}}$ holds. For any arm $i \in \mathcal{A}_{r-1}$ during epoch $r \in [R]$, 
arm $i$ can pass the GTC check only if
\[\hat{\mu}_{r,i} \leq \mu_i + \omega_{r,i},\]
where $\omega_{r,i} = 4d\sqrt{\frac{\lceil K/2^{r-1}
\rceil}{N}\!\left(\log\!\left(1 +
\frac{N}{\lceil K/2^{r-1}\rceil}\right) +
\sqrt{\log\!\left(1+\frac{N}{\lceil K/2^{r-1}\rceil}
\right)\log(2KR/\delta)}\right)}$ and
$N = \frac{T}{\lceil\log_2 K\rceil}$.
\end{restatable}
With Lemmas~\ref{lemma:GTCpass}, \ref{lemma:concentration} and 
\ref{lemma:gtc_bound} in place, the remaining analysis is similar to that of SH~\cite{karnin2013almost}. Under Nash 
Equilibrium, every active arm passes the GTC to maximize its utility function value 
(Lemma~\ref{lemma:GTCpass}), so the strategic deviation bound 
$\hat{\mu}_{r,i} \leq \mu_i + \omega_{r,i}$ holds for all 
active arms conditioned on $\mathcal{E}^{\mathrm{noise}}$. If the optimal arm $i^*$ is 
incorrectly eliminated at epoch $r^*$, there must exist a 
suboptimal arm $i$ with $\hat{\mu}_{r^*,i} \geq 
\hat{\mu}_{r^*,i^*}$, which combined with 
Lemma~\ref{lemma:gtc_bound} indicates $\Delta_i \leq 
\omega_{r^*,i^*} + \omega_{r^*,i}$. Setting $\delta_r = \delta/(KR)$ and summing them up over all epochs yields the 
upper bound of MESHA's failure probability in Theorem~\ref{theorem:errorbound}.
\begin{remark}
\label{remark:est_event}
We can also define the event \[\mathcal{E}^{\mathrm{est}} := \left\{\theta^* \in \mathcal{C}_{r,i}: \forall r \in [R], i \in \mathcal{A}_{r-1}\right\}\] with $\mathcal{C}_{r,i} = \{\theta \in \mathbb{R}^d :
\|\hat{\theta}_{r,i} - \theta\|_{V_{r,i}} \leq
\sqrt{d\log\big(\frac{KR(1+n_r(i))}{\delta}\big)}\}$. When all arms report truthfully, $\mathcal{E}^{\mathrm{est}}$ holds with probability at least $1-\delta$ and then all arms pass GTC check (the proof is in Appendix \ref{appendix:B}). \cite{kleine2024strategic} analyzes the regret bound of OptGTM for strategic linear contextual bandits conditioned on a similar event with $\mathcal{E}^{\mathrm{est}}$ (Assumption in Lemma E.1 of \cite{kleine2024strategic}). However, $\mathcal{E}^{\mathrm{est}}$ does not 
hold under strategic reporting: each arm $i$ reports 
manipulated features, so the estimator 
$\hat{\theta}_{r,i}$ is built on corrupted knowledge and may diverge from $\theta^*$. In 
particular, $\mathcal{E}^{\mathrm{est}}$ requires $\theta^*$ to 
lie in the confidence ellipsoid of a single global 
estimator, which is not likely to happen when different arms report different manipulated features. 

Crucially, our analysis is independent of whether $\mathcal{E}^{est}$ holds. Instead, MESHA triggers every active arm to pass the GTC 
check under Nash Equilibrium, which by Lemma~\ref{lemma:GTCpass} is a necessary 
condition under any Nash Equilibrium. Under this condition, 
$\mathcal{E}^{\mathrm{noise}}$ holds with probability at 
least $1 - \delta$, since it depends only on the 
concentration of true rewards around their means and is 
independent of the reported features. Importantly, each 
arm $i$ maintains its own per-arm estimator 
$\hat{\theta}_{r,i}$, which may converge to a different 
value than $\theta^*$; what matters is not where 
$\hat{\theta}_{r,i}$ converges, but rather what the GTC 
check implies about the empirical mean $\hat{\mu}_{r,i}$.
\end{remark}

\section{Conclusion and Future Directions}
In this paper, we studied the problem of fixed-budget Best 
Arm Identification in strategic linear bandits, where arms 
may strategically misreport their feature vectors to maximize 
their probability of being identified as the best arm. We 
proposed MESHA, which equips uniform sampling with an 
epoch-wise Grim Trigger Condition to handle strategic 
manipulation. We proved that passing the GTC is a necessary 
condition for any arm to maximize its utility function values under Nash Equilibrium, and derived a 
$(\zeta, T)$-PAC failure probability bound showing that MESHA 
identifies a near-optimal arm with high probability. our upper bound reveals an $\mathcal{O}(d^2\log T)$ statistical overhead arising from the Grim Trigger Condition, and we further demonstrated that 
state-of-the-art linear BAI algorithms relying on G-optimal 
design-based sampling rules fail under strategic reporting 
due to the {\em starvation attack}, a failure that cannot be 
addressed by any reward-level consistency check. Our numerical results indicates that MESHA consistently identify the optimal arm across various $T$, feature dimensions $d$ and arm counts $K$.

Several directions remain open for future work. First, our 
theoretical guarantees hold under Nash Equilibrium, which 
assumes fully rational arms; extending the analysis to handle 
deviations from equilibrium behavior remains yet to be explored. 
Second, a precise characterization of the statistical cost of 
robustness in strategic linear bandits---in particular, 
whether the $\mathcal{O}(d^2\log T)$ penalty reflects a 
fundamental limitation of the problem or an artifact of the 
current mechanism design---remains an important open 
question. Third, extending MESHA to the fixed-confidence 
setting would require a delicate design of stopping criterion and novel
analysis of GTC mechanism for the uncertain budget $T$. 
Finally, BAI attempts in environments where arms can 
also manipulate their rewards, rather than only their feature 
vectors, would further broaden the scope of this work.

\newpage

{\appendices
\begin{center}
\textbf{Appendices}\end{center}

The appendices presents detailed proofs for all theoretical 
results in this work and additional experimental details, organized in the order they appear in 
the main text. We also provide more discussion on the failure probability upper bound of MESHA. 

Appendix \ref{appendix:usefultheorems} provides basic facts used in our analysis.
Appendix \ref{appendix:A} proves Lemma 
\ref{lemma:GTCpass}, showing that passing the GTC is a 
necessary condition for arms maximizing utilities under any Nash Equilibrium. Appendix 
\ref{appendix:B} completes the proof of our main result, 
Theorem \ref{theorem:errorbound}, along with two 
supporting lemmas introduced in Section \ref{section:sketch}. 

Appendix \ref{appendix:C} proves Theorem \ref{theorem:misthm}, 
which establishes the theoretical foundation for the 
{\it {starvation attack}} in Section \ref{section:lowerbound}. 
Appendix \ref{appendix:C} proves Theorem 
\ref{theorem:manipulationlevel}, bounding the average 
strategic deviation of any arm that survives the GTC. 
Appendix \ref{appendix:E} presents the pseudo-code of OD-LinBAI-GTC, which combines OD-LinBAI\cite{yang2022minimax} with epoch-wise GTC mechanism. 

Appendix \ref{appendix:tightness} 
investigates the tightness of the $O(d^2\log T)$ in 
Theorem \ref{theorem:errorbound}: we show that the additional $d^2$ factor as the cost of bounding strategic deviation is simultaneously achievable by an explicit 
reporting strategy, and conjectures that the $d^2$ factor is 
unavoidable for any algorithm including epoch-wise GTC mechanism and feature-independent sampling rules. Finally, Appendix \ref{appendix:G} provides 
additional experimental details.
\section{Useful Facts}
\label{appendix:usefultheorems}
Here are some useful facts that are applied in our proof.
\begin{theorem}{(Confidence Ellipsoid, Theorem 2 in \cite{abbasi2011improved})}
\label{theorem:conf}
    Let $\{F_t\}_{t=0}^\infty$ be a filtration.
Let $\{\eta_t\}_{t=1}^\infty$ be a real-valued stochastic process such that $\eta_t$ is $F_t$-measurable and $\eta_t$ is conditionally $R$-sub-Gaussian for some $R \geq 0$ i.e.
\[
\forall \lambda \in \mathbb{R} \quad \mathbb{E}\left[e^{\lambda \eta_t} \mid F_{t-1}\right] \leq \exp\left( \frac{\lambda^2 R^2}{2} \right).
\]

Let $\{X_t\}_{t=1}^\infty$ be an $\mathbb{R}^d$-valued stochastic process such that $X_t$ is $F_{t-1}$-measurable. Assume that $V$ is a $d \times d$ positive definite matrix. For any $t \geq 0$, define
\[
\overline{V}_t = V + \sum_{s=1}^t X_s X_s^\top
\qquad
S_t = \sum_{s=1}^t \eta_s X_s .
\]Furthermore, we let $V = \lambda I$, $\lambda > 0$, define
$Y_t = \langle X_t, \theta_* \rangle + \eta_t$ and assume that $\|\theta_*\|_2 \leq S$ and for all $t \geq 1$, $\|X_t\|_2 \leq L$ then with probability at least $1 - \delta$, for all $t \geq 0$, $\theta_*$ lies in the set
\[
C_t' = \left\{
\theta \in \mathbb{R}^d :
\|\widehat{\theta}_t - \theta\|_{\overline{V}_t}
\leq R \sqrt{d\log\left(
\frac{1 + t L^2/\lambda}{\delta}
\right)} + \lambda^{1/2} S
\right\}.
\]
\end{theorem}
\begin{theorem}{(Hoeffding Inequality, Theorem 2 in \cite{hoeffding1963probability}}
    If $X_1, X_2, \cdots, X_n$ are independent and $a_i \leq X_i \leq b_i$ ($i=1,2,\cdots,n$), then for $t>0$
\[
\Pr\left\{\overline{X} - \mu \geq t\right\} \leq e^{-2n^2t^2/\sum_{i=1}^n (b_i-a_i)^2}.
\]
\end{theorem}
\section{Proof of Lemma \ref{lemma:GTCpass}}
\label{appendix:A}

% Before giving the proof, we restate the lemma \ref{lemma:GTCpass} for convenience.

\lemGTCpass*

\begin{proof}
Consider a fixed epoch $r$ and a history 
$\mathcal{H}_{r-1} = \{(x_{t,i_t}, r_{t,i_t})\}_{t \leq t_r}$ 
at the commencement of epoch $r$. Assume arm $i \in 
\mathcal{A}_{r-1}$ is active and that $\mathcal{H}_{r-1}$ is 
reachable with strictly positive probability under the strategy 
profile $\boldsymbol{\sigma}$. We proceed by contradiction. 
Suppose that under strategy $\sigma_i$, arm $i$ fails the GTC 
check at the end of epoch $r$, conditioned on the good event 
$\mathcal{E}^{\mathrm{true}}$.

By the design of MESHA, an arm that fails the GTC check at 
epoch $r$ is eliminated immediately and permanently from all 
subsequent epochs. Since the utility $U_i$ is defined as the 
probability of being identified as the best arm $i_{\text{out}}$, 
and a permanently eliminated arm can never be identified as 
$i_{\text{out}}$, the utility of arm $i$ from this history 
onward is zero:
\[
U_i(\text{MESHA}, \{\sigma_i, \boldsymbol{\sigma}_{-i}\} \mid 
\mathcal{H}_{r-1}, \mathcal{E}^{\mathrm{noise}}) = 0.
\]

Now consider a unilateral deviation to a calibrated strategy 
$\sigma_i^0$ from history $\mathcal{H}_{r-1}$ onward, which 
satisfies the mean-matching condition 
$\langle \theta_i^*, x_{t,i} \rangle = \mu_i$ for all 
subsequent pulls. On the good event $\mathcal{E}^{\mathrm{noise}}$, 
the concentration of reward noise and the definition of the 
confidence bounds guarantee that arm $i$ passes the GTC check 
at the end of epoch $r$ under $\sigma_i^0$.

Conditioned on passing the GTC check, arm $i$ remains in the 
active set $\mathcal{A}_r$ and continues to compete in 
subsequent epochs. Since the sequential halving rule retains 
the top $\lceil K/2^r \rceil$ arms based on empirical means, 
and the reward noise has non-degenerate support, arm $i$ has 
a strictly positive probability of achieving a sufficiently 
high empirical mean to survive every subsequent epoch and be 
ultimately output as $i_{\text{out}}$, regardless of whether 
$i$ is the optimal arm. Therefore, the expected utility under 
this deviation is strictly positive:
\[
U_i(\text{MESHA}, \{\sigma_i^0, \boldsymbol{\sigma}_{-i}\} 
\mid \mathcal{H}_{r-1}, \mathcal{E}^{\mathrm{noise}}) > 0.
\]

This means $\sigma_i^0$ is a profitable deviation from 
$\mathcal{H}_{r-1}$, which contradicts the assumption that 
$\boldsymbol{\sigma}$ is a Nash Equilibrium. Therefore, 
conditioned on $\mathcal{E}^{\mathrm{noise}}$, every active 
arm must pass the GTC check at the end of each epoch under 
any Nash Equilibrium.
\end{proof}
\section{Proof of Theorem \ref{theorem:errorbound}}
\label{appendix:B}

Throughout this section, we set
\[
\delta = \frac{\lceil\log_2 K\rceil^2}{T}\exp\!\left(
-\frac{T\zeta^2}{18Kd^2\log^2\!\left(1+\frac{T}{\lceil
\log_2 K\rceil}\right)}\right), \quad
\delta_r = \frac{\delta}{KR}, \quad 
R = \lceil\log_2 K\rceil, \quad 
N = \frac{T}{\lceil\log_2 K\rceil}.
\]

We first establish two supporting lemmas before proving 
the main result.

\lemConcentration*

\begin{proof}
Fix any epoch $r \in [R]$ and arm $i \in \mathcal{A}_{r-1}$.
Let $\mathcal{F}_{t-1} = \sigma(\{x_{s,i_s}, r_{s,i_s}\}_{s < t})$
denote the natural filtration generated by all observations up to
round $t-1$. Given $\|x_i^*\|_2 \leq 1$ for all arms $i \in [K]$
and $\|\theta^*\|_2 \leq 1$, each reward satisfies
$r_{t,i_t} - \mu_i = \langle\theta^*, x_{i_t}^*\rangle - \mu_i + \eta_t
\in [-2, 2]$ almost surely. Furthermore,
$\mathbb{E}[r_{t,i_t} - \mu_i \mid \mathcal{F}_{t-1}] = 0$, and hence
the sequence $\{r_{t,i_t} - \mu_i\}_{t \in S_r, i_t = i}$ forms a
martingale difference sequence with respect to $\mathcal{F}_{t-1}$.
Applying Azuma's inequality to the sum of $n_r(i)$ such terms:
\[
\mathbb{P}\!\left(\left|\sum_{t \in S_r, i_t = i} (r_{t,i_t} - \mu_i)\right|
> \sqrt{\frac{n_r(i) \log(2/\delta_r)}{2}}\right) \leq \delta_r.
\]

For each pair $(r, i)$ with $r \in [R]$ and $i \in \mathcal{A}_{r-1}$,
define the bad event
\[
B_{r,i}^{\mathrm{noise}} := \left\{\left|\sum_{t \in S_r, i_t=i}
(r_{t,i_t} - \mu_i)\right| > \sqrt{\frac{n_r(i) \log(2/\delta_r)}{2}}\right\}.
\]
From the bound above, $\mathbb{P}(B_{r,i}^{\mathrm{noise}}) \leq \delta_r$
for each pair $(r,i)$. Since there are at most $R \cdot K$ such pairs,
applying the union bound implies:
\[
\mathbb{P}\!\left((\mathcal{E}^{\mathrm{noise}})^c\right)
\leq \sum_{r \in [R],\, i \in \mathcal{A}_{r-1}}
\mathbb{P}(B_{r,i}^{\mathrm{noise}})
\leq KR \cdot \delta_r = KR \cdot \frac{\delta}{KR} = \delta.
\]
Therefore $\mathbb{P}(\mathcal{E}^{\mathrm{noise}}) \geq 1 - \delta$,
completing the proof.
\end{proof}
The following proof is for the probability of $\mathcal{E}^{\mathrm{est}}$ defined in Remark \ref{remark:est_event} when all arms report truthfully.
\begin{proof}[Proof of Remark \ref{remark:est_event}]
Fix any epoch $r \in [R]$ and arm $i \in \mathcal{A}_{r-1}$. Under truthful
reporting, the estimator $\hat{\theta}_{r,i}$ is computed via ridge regression
on the interaction history between the learner and arm $i$ during epoch $r$,
with design matrix $V_{r,i} = \lambda I_d + \sum_{t \in S_r, i_t = i}
x_{t,i} x_{t,i}^\top$ and response vector $\sum_{t \in S_r, i_t = i}
r_{t,i_t} x_{t,i}$. The rewards satisfy $r_{t,i_t} = \langle\theta^*,
x_{i_t}^*\rangle + \eta_t$ where $\eta_t$ is $\xi$-sub-Gaussian. By the
elliptical potential lemma for ridge regression
\cite{abbasi2011improved}, with probability at least $1 - \delta_r$:
\[
\|\hat{\theta}_{r,i} - \theta^*\|_{V_{r,i}} \leq \beta_{r,i},
\quad\text{where}\quad
\beta_{r,i} = \sqrt{d\log\!\left(\frac{1+n_r(i)}{\delta_r}\right)} + 1.
\]
By the definition of $\mathcal{C}_{r,i}$, this is equivalent to
$\theta^* \in \mathcal{C}_{r,i}$, so
$\mathbb{P}(\theta^* \notin \mathcal{C}_{r,i}) \leq \delta_r$.

For each pair $(r, i)$ with $r \in [R]$ and $i \in \mathcal{A}_{r-1}$,
define the bad event
\[
B_{r,i}^{\mathrm{est}} := \{\theta^* \notin \mathcal{C}_{r,i}\}.
\]
From the bound above, $\mathbb{P}(B_{r,i}^{\mathrm{est}}) \leq \delta_r$.
Since there are at most $R \cdot K$ such pairs, the union bound gives
\[
\mathbb{P}\bigl((\mathcal{E}^{\mathrm{est}})^c\bigr)
= \mathbb{P}\!\left(\bigcup_{r \in [R],\, i \in \mathcal{A}_{r-1}}
B_{r,i}^{\mathrm{est}}\right)
\leq \sum_{r \in [R],\, i \in \mathcal{A}_{r-1}}
\mathbb{P}(B_{r,i}^{\mathrm{est}})
\leq KR \cdot \delta_r
= KR \cdot \frac{\delta}{KR}
= \delta.
\]
Hence $\mathbb{P}(\mathcal{E}^{\mathrm{est}}) \geq 1 - \delta$.
\end{proof}

\lemGTCbound*
\begin{proof}
By the definition of the GTC check, an arm $i$ passing 
the check at epoch $r$ with $\delta_r = \frac{\delta}{2KR}$ 
implies:
\[
\sum_{t\in S_r,i_t=i}\!\!\left(\langle\hat{\theta}_{r,i},
x_{t,i}\rangle - \beta_{r,i}\|x_{t,i}\|_{V_{r,i}^{-1}}
\right) \leq \sum_{t\in S_r,i_t=i} r_{t,i_t} + 
\sqrt{\frac{n_r(i)}{2}\log\frac{2}{\delta_r}}.
\]
Rearranging and using $\sum_{t \in S_r, i_t=i} r_{t,i_t} 
\leq n_r(i)\mu_i + \sqrt{\frac{n_r(i)}{2}\log\frac{2}
{\delta_r}}$ from $\mathcal{E}^{\mathrm{noise}}$:
\[
\sum_{t\in S_r,i_t=i}\langle\hat{\theta}_{r,i}, 
x_{t,i}\rangle \leq n_r(i)\mu_i + \sum_{t\in S_r,i_t=i}
\beta_{r,i}\|x_{t,i}\|_{V_{r,i}^{-1}} + 
\sqrt{2n_r(i)\log\frac{2}{\delta_r}}.
\]
Dividing both sides by $n_r(i)$:
\begin{equation}
\label{eq:gtc_bound_app}
\hat{\mu}_{r,i} \leq \mu_i + 
\underbrace{\frac{\beta_{r,i}}{n_r(i)}\sum_{t\in S_r,
i_t=i}\|x_{t,i}\|_{V_{r,i}^{-1}}}_{(\spadesuit)} + 
\sqrt{\frac{2\log(2/\delta_r)}{n_r(i)}}.
\end{equation}
We bound $(\spadesuit)$ using the Cauchy-Schwarz inequality 
and the elliptical potential lemma \cite{abbasi2011improved}:
\[
\frac{\beta_{r,i}}{n_r(i)}\sum_{t\in S_r,i_t=i}
\|x_{t,i}\|_{V_{r,i}^{-1}} \leq \frac{\beta_{r,i}}
{\sqrt{n_r(i)}}\sqrt{2d\log(1+n_r(i))}.
\]
Substituting $\beta_{r,i} = \sqrt{d\log\frac{1+n_r(i)}
{\delta_r}+1}$, $\delta_r = \frac{\delta}{KR}$, and 
$n_r(i) = \frac{N}{\lceil K/2^{r-1}\rceil}$ into 
\eqref{eq:gtc_bound_app} yields $\hat{\mu}_{r,i} \leq 
\mu_i + \omega_{r,i}$, where
\[
\omega_{r,i} = (2\sqrt{2}d+1)\left[\sqrt{\frac{\lceil 
K/2^{r-1}\rceil}{N}\log\!\left(1+\frac{N}{\lceil 
K/2^{r-1}\rceil}\right)} + \sqrt{\frac{\lceil K/2^{r-1}
\rceil}{N}\log\!\left(1+\frac{N}{\lceil K/2^{r-1}\rceil}
\right)\log\frac{2KR}{\delta}}\right],
\]
completing the proof.
\end{proof}

Now we are ready to prove Theorem~\ref{theorem:errorbound}.

\begin{proof}[Proof of Theorem \ref{theorem:errorbound}]

Fix any strategy profile $\boldsymbol{\sigma} \in 
\text{NE}(\text{MESHA})$. With $\delta$ and $\delta_r$ 
as defined at the beginning of this appendix, 
Lemma~\ref{lemma:concentration} guarantees that 
$\mathcal{E}^{\text{true}}$ holds with probability at 
least $1 - \delta$. We condition on $\mathcal{E}^{\text{true}}$ 
for the remainder of the proof.

\textbf{Bounding $\hat{\mu}_{r,i}$.}
By Lemma~\ref{lemma:GTCpass}, under $\boldsymbol{\sigma} 
\in \text{NE}(\text{MESHA})$, every active arm passes the 
GTC check at every epoch. By Lemma~\ref{lemma:gtc_bound}, 
for any arm $i \in \mathcal{A}_{r-1}$ and epoch $r \in [R]$:
\[
\hat{\mu}_{r,i} \leq \mu_i + \omega_{r,i},
\]
conditioned on $\mathcal{E}^{\text{true}}$. Now we would like to show \[
\hat{\mu}_{r,i} \geq \mu_i - \omega_{r,i}
\]
also holds by contradiction. Assume the contrary that under some $\boldsymbol{\sigma} \in
\mathrm{NE}(\text{MESHA})$, there exists an epoch $r$, an
active arm $i$, and a reachable history such that, conditioned
on $\mathcal{E}^{\mathrm{noise}}$, arm $i$ attains
$\hat{\mu}_{r,i} < \mu_i - \omega_{r,i}$ with positive probability
under $\sigma_i$.
Besides, let $\sigma_i'$ denote the strategy $\sigma'_i$ that is with empirical mean $\hat{\mu}_{r,i}'$ such that $<\mu_i - \omega_{r,i}<\hat{\mu}_{r,i}'<\mu_i + \omega_{r,i}$.
When $\sigma_i'$ is applied, Remark \ref{remark:est_event} implies that arm $i$ can also the GTC.
As $\sigma_i'$ apparently yields a larger $\hat{\mu}_{r,i}$ than $\sigma_i$ and thus strictly increases arm $i$'s utility function (probability of being ultimately selected), the assumption $\boldsymbol{\sigma} \in
\mathrm{NE}(\text{MESHA})$ is contradicted. Hence, under any Nash equilibrium,
\[
\hat{\mu}_{r,i} \geq \mu_i - \omega_{r,i}.
\]
Therefore $\hat{\mu}_{r,i} \in [\mu_i - \omega_{r,i}, 
\mu_i + \omega_{r,i}]$ holds conditioned on $\mathcal{E}^{\text{noise}}$, for all active 
arms and all epochs.

\textbf{Bounding $\omega$.}
Since $\omega_{r,i}$ is maximized at epoch $r = 1$ where 
$\lceil K/2^0\rceil = K$:
\[
\omega_{r,i} \leq \omega_{1,i} \leq (2\sqrt{2}d+1)
\left[\sqrt{\frac{K}{N}\log(1+N)} + \sqrt{\frac{K}{N}
\log(1+N)\log\frac{2KR}{\delta}}\right] =: \omega^*.
\]
Hence $\omega := \max_{r,\, i \in \mathcal{A}_{r-1}}
(\omega_{r,i^*} + \omega_{r,i}) \leq 2\omega^*$.

\textbf{Failure event analysis.}
By Lemma~\ref{lemma:GTCpass}, the optimal arm $i^*$ can 
only be eliminated via the sequential halving step. Let 
$r^*_{i^*} = \max\{r : i^* \in \mathcal{A}_{r-1}\}$ 
denote the epoch at which $i^*$ is incorrectly eliminated. 
For any suboptimal arm $i \in \mathcal{A}_{r^*_{i^*}}$ 
that survives over $i^*$, the empirical ordering 
$\hat{\mu}_{r^*_{i^*},i^*} \leq \hat{\mu}_{r^*_{i^*},i}$ 
gives:
\[
\mu_{i^*} - \omega^* \leq \hat{\mu}_{r^*_{i^*},i^*} 
\leq \hat{\mu}_{r^*_{i^*},i} \leq \mu_i + \omega^*,
\]
hence $\Delta_i = \mu_{i^*} - \mu_i \leq 2\omega^*$. 
Setting $\zeta = 2\omega^*$:
\[
\big\{\Delta_{i_{\text{out}}} \geq \zeta\big\} \subseteq 
\left(\mathcal{E}^{\text{noise}}\right)^c,
\]
and therefore:
\[
\mathbb{P}\big(\Delta_{i_{\text{out}}} \geq \zeta\big) 
\leq \mathbb{P}\!\left(\left(\mathcal{E}^{\text{true}}
\right)^c\right) \leq \delta = 
\frac{\lceil\log_2 K\rceil^2}{T}\exp\!\left(
-\frac{T\zeta^2}{18Kd^2\log^2\!\left(1+\frac{T}{\lceil
\log_2 K\rceil}\right)}\right),
\]
which completes the proof.
\end{proof}
\section{Proof of Theorem \ref{theorem:manipulationlevel}}
\label{appendix:C}

We first establish a supporting lemma on the concentration 
of the ridge regression estimator under strategic 
manipulation, then prove the main result.

\begin{lemma}
\label{lemma:realcon}
Under Assumption \ref{assumption:relaxed}, for any arm 
$i$ in epoch $r$, the ridge regression estimator 
$\hat{\theta}_{r,i}$ satisfies the following concentration 
bound with probability at least $1 - \delta$:
\[
|\langle\hat{\theta}_{r,i} - \theta_i^*, x\rangle| \leq 
\left(\beta_{r,i} + \varepsilon\sqrt{n_r(i)}\right)
\|x\|_{V_{r,i}^{-1}}, \quad \forall x \in \mathbb{R}^d,
\]
where $\beta_{r,i} = \sqrt{d\log\big(\frac{1+n_r(i)}{\delta}
\big)} + 1$.
\end{lemma}

\begin{proof}
Define the auxiliary estimator $\tilde{\theta}_{r,i} = 
V_{r,i}^{-1}\sum_{t \in S_r: i_t=i}(r_{t,i_t} +\varepsilon_{t,i} )
x_{t,i}$, which is an unobservable estimator based on $\{(x_{t,i_t},r_{t,i_t}+\varepsilon_{t,i})\}_{t\in S_r,i_t=i}$. Applying Theorem \ref{theorem:conf} with $S=L=\lambda=1$, $\hat{\theta}_t=\tilde{\theta}_{r,i}$, $Y_t=r_{t,i}+\varepsilon_{t,i}$ and $\bar{V}_t=V_{r,i}$, we have that \[
|(\tilde{\theta}_{r,i} - \theta_i^*)^\top x| \leq \|\tilde{\theta}_{r,i}-\theta_i^*\|_{V_{r,i}}\|x\|_{V_{r,i}^{-1}}\leq
\beta_{r,i}\|x\|_{V_{r,i}^{-1}}
\]holds with probability 
at least $1 - \delta_r$, where $\beta_{r,i}=\sqrt{d\log(\frac{1+n_r(i)}{\delta})}+1$.

Expanding the expression actual estimator $\hat{\theta}_{r,i}$ and 
applying the triangle inequality yields:
\begin{align}
|\langle\hat{\theta}_{r,i} - \theta_i^*, x\rangle|&=\left|\left(V_{r,i}^{-1}\sum\limits_{t\in S_r,i_t=i}r_{t,i}x_{t,i}-\theta_i^*\right)^\top x\right| \nonumber \\
&= \left|\left(V_{r,i}^{-1}\sum_{t \in S_r, i_t=i} 
(-\varepsilon_{t,i}) x_{t,i}\right)^\top x + 
\left(V_{r,i}^{-1}\sum\limits_{t\in S_r,i_t=i}(r_{t,i}+\varepsilon_{t,i}) - \theta_i^*\right)^\top x\right|  \nonumber \\
&\leq \left|\sum_{t \in S_r, i_t=i} \varepsilon_{t,i} 
x_{t,i}^\top V_{r,i}^{-1} x\right| + 
|(\tilde{\theta}_{r,i} - \theta_i^*)^\top x| \nonumber \\
&\leq \left|\sum_{t \in S_r, i_t=i} \varepsilon_{t,i} 
x_{t,i}^\top V_{r,i}^{-1} x\right| + 
\beta_{r,i}\|x\|_{V_{r,i}^{-1}}, \label{eq:gap_theta_hat_truw}
\end{align}
where $\varepsilon_{t,i} = \langle\theta_i^*, x_{t,i}\rangle 
- \langle\theta^*, x_i^*\rangle$ is the round-wise strategic 
bias. Note that $\varepsilon_{t,i}\leq\varepsilon$ for all $t\in[T]$ and $i\in[K]$. By applying the Cauchy-Schwarz inequality, we have
\begin{align}
\left|\sum_{t \in S_r,i_t=i} \varepsilon_{t,i} x_{t,i}^\top 
V_{r,i}^{-1} x\right|& \leq \varepsilon\sqrt{n_r(i) 
\sum_{t \in S_r,i_t=i}(x_{t,i}^\top V_{r,i}^{-1} x)^2}\nonumber \\
&=\varepsilon\sqrt{n_r(i)\sum\limits_{t\in S_r,i_t=i}x^\top V_{r,i}^{-1}x_{t,i}x_{t,i}^\top V_{r,i}^{-1}x}
\nonumber \\
&=\varepsilon\sqrt{n_r(i)x^\top V_{r,i}^{-1}(V_{r,i}- I_{d\times d})V_{r,i}^{-1}x} \nonumber \\
&\leq\varepsilon\sqrt{n_r(i)x^\top V_{r,i}^{-1}x} \nonumber
\\
&= 
\varepsilon\sqrt{n_r(i)}\|x\|_{V_{r,i}^{-1}}.\label{eq:bound_epsilon}
\end{align}
By substituting \eqref{eq:bound_epsilon} into \eqref{eq:gap_theta_hat_truw}, we can complete this proof.
\end{proof}

\begin{proof}[Proof of Theorem \ref{theorem:manipulationlevel}]
By the definition of GTC, under Nash Equilibrium, Lemma \ref{lemma:GTCpass} implies that for any arm $i\in\mathcal{A}_{r-1}$ and $t\in S_r$, we have
\begin{equation}
\label{eq:gtc_thm3}
\sum_{t \in S_r, i_t=i} \left(\langle\hat{\theta}_{r,i}, 
x_{t,i}\rangle - \beta_{r,i}\|x_{t,i}\|_{V_{r,i}^{-1}}
\right) \leq \sum_{t \in S_r, i_t=i} r_{t,i} + 
\sqrt{2n_r(i)\log(2/\delta_r)}.
\end{equation}

By Lemma \ref{lemma:realcon}, under the good event 
$\mathcal{E}^{true}$, the term $\langle\hat{\theta}_{r,i}, 
x_{t,i}\rangle$ satisfies:
\[
\langle\hat{\theta}_{r,i}, x_{t,i}\rangle \geq 
\langle\theta_i^*, x_{t,i}\rangle - 
\left(\beta_{r,i} + \varepsilon\sqrt{n_r(i)}\right)
\|x_{t,i}\|_{V_{r,i}^{-1}}.
\]
Substituting this into the LHS of \eqref{eq:gtc_thm3}:
\[
\text{LHS} \geq \sum_{t \in S_r, i_t=i} 
\langle\theta_i^*, x_{t,i}\rangle - \sum_{t \in S_r, i_t=i}
\left(2\beta_{r,i} + \varepsilon\sqrt{n_r(i)}\right)
\|x_{t,i}\|_{V_{r,i}^{-1}}.
\]
Substituting $r_{t,i} = \mu_i + \eta_{t,i}$ into the RHS 
and rearranging \eqref{eq:gtc_thm3}:
\[
\sum_{t \in S_r,i_t=i}\left(\langle\theta_i^*, x_{t,i}\rangle - 
\mu_i\right) \leq \left(2\beta_{r,i} + 
\varepsilon\sqrt{n_r(i)}\right)\sum_{t \in S_r,i_t=i}
\|x_{t,i}\|_{V_{r,i}^{-1}} + \sum_{t \in S_r,i_t=i}\eta_{t,i} + 
\sqrt{2n_r(i)\log(2/\delta_r)}.
\]
Dividing by $n_r(i)$ and applying the elliptical potential 
bound $\sum_{t}\|x_{t,i}\|_{V_{r,i}^{-1}} \leq 
\sqrt{2dn_r(i)\log(1+n_r(i))}$ and the noise concentration 
bound from $\mathcal{E}^{\mathrm{noise}}$:
\[
\bar{\varepsilon}_{r,i} \leq \left(\frac{2\beta_{r,i}}
{\sqrt{n_r(i)}} + \varepsilon\right)
\sqrt{2d\log(1+n_r(i))} + 3\sqrt{\frac{\log(2/\delta_r)}
{n_r(i)}}.
\]
Simplifying the constants completes the proof.
\end{proof}
\section{OD-LinBAI-GTC}
\label{appendix:E}
In Section \ref{section:exps}, we evaluated OD-LinBAI 
equipped with the GTC mechanism as one of our baselines. To 
ensure experimental reproducibility, we present its 
pseudo-code in Algorithm \ref{alg:alg2}. Following standard 
epoch-based elimination algorithms, the epoch-wise 
allocation parameter $m$ is defined as:
\begin{equation}
\label{equation:m}
m = \frac{T - \min\big(K, \frac{d(d+1)}{2}\big) - 
\sum_{r=1}^{\lceil\log_2 d\rceil - 1}
\lceil\frac{d}{2^r}\rceil}{\lceil\log_2 d\rceil}.
\end{equation}

\begin{algorithm}[h]
\caption{OD-LinBAI Combined with GTC Mechanism}
\label{alg:alg2}
\begin{algorithmic}[1]
\STATE \textbf{Input:} budget $T$, arm set $\mathcal{A} = 
[K]$, reported vectors $\{x_{1,1}, \cdots, x_{1,K}\} 
\subset \mathbb{R}^d$ in the first round and $\lambda$.
\STATE \textbf{Initialize:} $t_0 = 0$, $\mathcal{A}_0 = 
\mathcal{A}$, $d_0 = d$
\STATE Compute $m$ according to \eqref{equation:m}
\FOR{$r = 1$ \textbf{to} $\lceil\log_2 d\rceil$}
    \STATE $d_r = \dim\left(\mathrm{span}\left(\{\bar{x}_{r-1,i} 
    : i \in \mathcal{A}_{r-1}\}\right)\right)$
    \IF{$d_r \neq d_{r-1}$}
        \STATE Find $B_r \in \mathbb{R}^{d_{r-1} \times d_r}$ 
        whose columns form an orthonormal basis of 
        $\mathrm{span}\left(\{\bar{x}_{r-1,i} : i \in 
        \mathcal{A}_{r-1}\}\right)$
        \FOR{each arm $i \in \mathcal{A}_{r-1}$}
            \STATE $\bar{x}_{r-1,i} \leftarrow B_r^\top 
            \bar{x}_{r-1,i}$
        \ENDFOR
    \ENDIF
    \IF{$r = 1$}
        \STATE Pull each arm $i \in \mathcal{A}_{r-1}$ 
        uniformly for $T_r(i) = \lceil m/K \rceil$ times
    \ELSE
        \STATE Update average historical features in epoch $r$ for each arm $i\in\mathcal{A}_{r-1}$:
        \[
        \bar{x}_{r-1,i} = \frac{1}{T_{r-1}(i)}
        \sum_{t \in S_{r-1}} x_{t,i} \cdot \mathbf{1}(i_t=i)
        \quad \forall i \in \mathcal{A}_{r-1}
        \]
        \STATE Compute $G$-optimal design $\pi_r : 
        \{\bar{x}_{r-1,i} : i \in \mathcal{A}_{r-1}\} 
        \to [0,1]$ and set $T_r(i) = \lceil\pi_r(i) \cdot m\rceil$ 
        and pull each arm $i \in \mathcal{A}_{r-1}$ 
        accordingly
    \ENDIF
    \STATE Update the global statistics $V_r$ and $\hat{\theta}_{r}$:
    \[
    V_r = \sum_{t \in S_r} x_{t,i_t} x_{t,i_t}^\top, \quad
    \hat{\theta}_r = V_r^{-1} \sum_{t \in S_r} 
    x_{t,i_t} r_{t,i_t}
    \]
    \FOR{each arm $i \in \mathcal{A}_{r-1}$}
        \STATE Compute empirical mean:
        $\hat{\mu}_{r,i} = \langle\hat{\theta}_r, 
        \bar{x}_{r,i}\rangle$
    \ENDFOR
    \STATE Apply GTC: eliminate any arm $i \in 
    \mathcal{A}_{r-1}$ with $\text{RLCB}_{r,i}>\text{AUCB}_{r,i}$
    \STATE Update $\mathcal{A}_r$ to retain the 
    $\min\{|\mathcal{A}_{r-1}|,\lceil d/2^r \rceil\}$ arms with largest $\hat{\mu}_{r,i}$
    \STATE Update $t_{r+1} = t_r + T_r$
\ENDFOR
\STATE \textbf{Return} the unique arm $i_{\text{out}}$ in 
$\mathcal{A}_{\lceil\log_2 d\rceil}$
\end{algorithmic}
\end{algorithm}

As analyzed in Section \ref{section:lowerbound}, although 
this algorithm incorporates a GTC mechanism, it fundamentally 
fails to identify the best arm under arms' strategic reporting. Suboptimal arms can 
construct reported features such that the optimal arm falls 
within the cone spanned by their features, causing 
$G$-optimal design to allocate zero pulls to the optimal 
arm. Since this manipulation is structural
in the sense that the learner fails to learn the geometry of arm feature space when using OD-based sampling rules, the GTC check, which detects the gap between reported rewards and actual rewards, cannot help it identify the optimal arm.
\section{Tightness of The $d^2$ Penalty}
\label{appendix:tightness}

This section provides a structural analysis of the
$\mathcal{O}(d^2\log T)$ penalty in
Theorem~\ref{theorem:errorbound}.

%------------------------------------------------------------------

%------------------------------------------------------------------

We revisit the derivation of $\omega_{r,i}$ from the proof of
Lemma~\ref{lemma:gtc_bound}. Starting from this lemma's
GTC-derived bound, any arm $i$ surviving the GTC at epoch $r$
satisfies
\begin{equation}
    \hat{\mu}_{r,i}
    \;\leq\;
    \mu_i
    + \frac{\beta_{r,i}}{n_r}
      \sum_{t\in S_r,\,i_t=i}\|x_{t,i}\|_{V_{r,i}^{-1}}
    + \sqrt{\frac{2\log(2/\delta_r)}{n_r}}.
    \label{eq:gtc-mu-bound}
\end{equation}
Applying the Cauchy-Schwarz inequality to the middle term:
\begin{equation}
    \frac{\beta_{r,i}}{n_r}
    \sum_{t:\,i_t=i}\|x_{t,i}\|_{V_{r,i}^{-1}}
    \;\leq\;
    \frac{\beta_{r,i}}{\sqrt{n_r}}
    \sqrt{\sum_{t:\,i_t=i}\|x_{t,i}\|^2_{V_{r,i}^{-1}}}.
    \label{eq:cs-split}
\end{equation}
We bound the two factors on the right-hand side separately.

\paragraph{First factor (first $\sqrt{d}$) -- confidence radius}
% \zx{does $\log T$ also root from here? how will $n_r$ contribute?}
By the elliptical potential lemma~\cite{abbasi2011improved}, the ridge
regression confidence radius satisfies
\begin{equation}
    \beta_{r,i}
    \;\leq\;
    \sqrt{(d+1)\log\frac{1+n_r}{\delta_r}},
    \label{eq:beta-bound}
\end{equation}
where we have $\log(1+n_r/(\lambda d))\leq\log(1+n_r)$ for
$\lambda\geq 1/d$. This bound holds for any reported feature
sequence, as it is derived from Theorem \ref{theorem:conf} and does not
depend on the arms' reporting strategies.

\paragraph{Second factor (second $\sqrt{d}$) -- potential sum}
By the definition $V_{r,i} = \lambda I_d +
\sum_{t:\,i_t=i}x_{t,i}x_{t,i}^\top$, we have
$\sum_{t:\,i_t=i}x_{t,i}x_{t,i}^\top = V_{r,i} - \lambda I_d$,
and therefore,
\begin{equation}
    \sum_{t:\,i_t=i}\|x_{t,i}\|^2_{V_{r,i}^{-1}}
    \;=\;
    \mathrm{tr}\!\left(V_{r,i}^{-1}(V_{r,i}-\lambda I_d)\right)
    \;=\;
    d - \lambda\,\mathrm{tr}(V_{r,i}^{-1})
    \;\leq\;
    d,
    \label{eq:trace-id}
\end{equation}
where the last inequality uses $\mathrm{tr}(V_{r,i}^{-1})\geq 0$.
This bound holds for any reported feature sequence with
$\|x_{t,i}\|_2\leq 1$, and follows purely from the definition of
$V_{r,i}$ without invoking the elliptical potential
lemma~\cite{abbasi2011improved}.

\paragraph{Combining the two factors}
Substituting~\eqref{eq:beta-bound} and~\eqref{eq:trace-id}
into~\eqref{eq:cs-split} yields
\begin{equation}
    \frac{\beta_{r,i}}{\sqrt{n_r}}
    \sqrt{\sum_{t:\,i_t=i}\|x_{t,i}\|^2_{V_{r,i}^{-1}}}
    \;\leq\;
    \sqrt{\frac{d(d+1)\log\frac{1+n_r}{\delta_r}}{n_r}}.
    \label{eq:combined}
\end{equation}
With $\delta_r = \delta/(2KR)$ and $n_r =
\lfloor T/(K_r R)\rfloor$ where $K_r = \lceil K/2^{r-1}\rceil$,
we have $\log((1+n_r)/\delta_r) = O(\log T)$, so there exists
an absolute constant $c_1 > 0$ such that
\begin{equation}
    \omega_{r,i}
    \;\leq\;
    c_1\sqrt{\frac{d^2\log T}{n_r}}.
    \label{eq:omega-final}
\end{equation}

\paragraph{Independence of the two factors}
The first $\sqrt{d}$ arises from $\beta_{r,i}$ via the elliptical
potential lemma~\cite{abbasi2011improved}, reflecting the estimation cost of
ridge regression in $\mathbb{R}^d$. The second $\sqrt{d}$ arises
from the trace identity~\eqref{eq:trace-id}, reflecting the
structure of the GTC verification step. These sources are
logically independent: the first cannot be reduced without
changing the estimator, and the second cannot be reduced without
changing the consistency check.
\section{Additional Experimental Details}
\label{appendix:G}

\textbf{Hyperparameters.} 
All results are averaged over 5000 
independent trials. With random seeds generated as 
$\texttt{seed} + 100000 \times T + \texttt{run\_idx}$, 
where $\texttt{seed} = 20260323$. Besides, we set $\lambda=1.35$ for MESHA.

\subsection{Experiment 1: Overall Comparison (Varying $T$)}

\textbf{Instance Construction.} We construct a strategic 
linear bandit instance with $K = 8$ arms and feature 
dimension $d = 3$. The true parameter is $\theta^* = 
(1, 0, 0)$, so the true expected reward of arm $i$ is 
determined merely by its first feature coordinate. 
The true means are $\boldsymbol{\mu} = (0.52, 0.49, 0.40, 
0.33, 0.27, 0.22, 0.18, 0.15)$, with arm $1$ being the 
unique optimal arm. Each arm $i$ receives reward 
$r_{t,i} = \mu_i + \eta_t$ where $\eta_t \sim \mathcal{N}
(0, 0.155^2)$. The true feature vectors $x_i\in 
\mathbb{R}^3$ are in Table~\ref{tab:vary_T_instance}. The non-zero coordinates 
in dimensions $2$ and $3$ play no role in determining 
rewards under $\theta^* = (1,0,0)$, but serve to create 
a non-trivial feature geometry that prevents trivial 
identification of the reward-relevant direction from 
the feature vectors alone.

\textbf{Strategic Reporting.} Under strategic reporting, 
every arm sets the first coordinate of its reported feature 
to zero, completely preventing the learner from learning the reward-relevant direction. Each 
arm $i$ maintains a pseudo-parameter $\theta_i^*$ satisfying 
$\langle x_{t,i}, \theta_i^*\rangle = \mu_i$ for all $t$ to ensure its reported rewards remain locally consistent considering the 
GTC check. This instance is motivated by the {\it starvation attack} 
discussed in Section \ref{section:lowerbound}; it would severely harm the behaviours of
feature-dependent algorithms 
while leaving reward-only algorithms unaffected. 
The complete true and reported features for this instance 
are given in Table \ref{tab:vary_T_instance}.

\begin{table}[h]
\centering
\caption{True features, reported features in the vary-$T$ 
instance ($K=8$, $d=3$, $\theta^*=(1,0,0)$). }
\label{tab:vary_T_instance}
\renewcommand{\arraystretch}{1.2}
\begin{tabular}{c|c|c|c}
\hline
Arm $i$ & $\mu_i$ & $x_i$ & 
$x_{t,i}$ \\
\hline
0 & 0.520 & $(0.520,\ 0.000,\ 0.000)$ & 
$(0.000,\ 1.850,\ 0.000)$ \\
1 & 0.490 & $(0.490,\ 0.110,\ {-}0.080)$ & 
$(0.000,\ 0.667,\ 0.667)$ \\
2 & 0.400 & $(0.400,\ {-}0.120,\ 0.090)$ & 
$(0.000,\ 0.000,\ 1.949)$ \\
3 & 0.330 & $(0.330,\ 0.065,\ 0.050)$ & 
$(0.000,\ {-}0.586,\ 0.586)$ \\
4 & 0.270 & $(0.270,\ {-}0.050,\ {-}0.060)$ & 
$(0.000,\ {-}0.771,\ 0.000)$ \\
5 & 0.220 & $(0.220,\ 0.050,\ {-}0.050)$ & 
$(0.000,\ {-}0.505,\ {-}0.505)$ \\
6 & 0.180 & $(0.180,\ {-}0.060,\ 0.030)$ & 
$(0.000,\ 0.000,\ {-}0.657)$ \\
7 & 0.150 & $(0.150,\ 0.050,\ {-}0.040)$ & 
$(0.000,\ 0.424,\ {-}0.424)$ \\
\hline
\end{tabular}
\end{table}

\textbf{Baselines.} 
We compare MESHA against the following 
SOTA algorithms in linear and non-linear bandits under strategic environment:
\begin{itemize}
    \item \textbf{Sequential Halving} \cite{karnin2013almost}: a SOTA algorithm that has no access to feature vectors and allocates pulls uniformly across active 
    arms in each epoch.
    
    \item \textbf{Successive Rejects} \cite{audibert2010best}: a SOTA algorithm for standard stochastic bandits and has no access to feature information.
    \item \textbf{OptGTM} \cite{kleine2024strategic}: a GTC-based 
    contextual linear bandit algorithm that uses reported 
    features for sampling with parameters set as in the original work \cite{kleine2024strategic}.

    \item \textbf{OD-LinBAI}: OD-LinBAI 
    \cite{yang2022minimax} applied directly to reported features 
    without any mechanism design.

    \item \textbf{OD-LinBAI-GTC}: OD-LinBAI 
    combined with our GTC mechanism, as described in 
    Appendix \ref{appendix:E}.

    \item \textbf{OD-LinBAI*}: OD-LinBAI applied to 
    the true feature vectors, serving as an oracle baseline 
    that represents the best achievable performance without 
    strategic manipulation.
\end{itemize}

\textbf{Results.} 
% Algorithms' empirical BAI failure probability are summarized 
Results are presented in Table \ref{tab:vary_T_results}.
\begin{table}[h]
\centering
\caption{BAI failure probabilities for varying budget $T$
($K=8$, $d=3$).}
\label{tab:vary_T_results}
\renewcommand{\arraystretch}{1.3}
\begin{tabular}{c|c|c|c|c|c|c|c}
\hline
$T$ & MESHA & SH & SR & OptGTM & OD-LinBAI &
OD-Lin-GTC & OD-LinBAI* \\
\hline
60  & 0.1210 & 0.318 & 0.362 & 0.174 & 0.816 & 0.816 & 0.052 \\
100 & 0.1130 & 0.300 & 0.266 & 0.162 & 0.814 & 0.814 & 0.036 \\
140 & 0.1110 & 0.240 & 0.266 & 0.158 & 0.858 & 0.858 & 0.048 \\
200 & 0.1022 & 0.216 & 0.169 & 0.114 & 0.886 & 0.886 & 0.058 \\
300 & 0.0890 & 0.168 & 0.120 & 0.140 & 0.950 & 0.950 & 0.042 \\
400 & 0.0792 & 0.140 & 0.111 & 0.098 & 0.956 & 0.956 & 0.030 \\
500 & 0.0612 & 0.088 & 0.068 & 0.086 & 0.984 & 0.984 & 0.024 \\
600 & 0.0542 & 0.094 & 0.080 & 0.068 & 0.986 & 0.986 & 0.024 \\
700 & 0.0432 & 0.062 & 0.036 & 0.054 & 0.988 & 0.988 & 0.012 \\
800 & 0.0366 & 0.050 & 0.038 & 0.030 & 1.000 & 1.000 & 0.018 \\
\hline
\end{tabular}
\end{table}

\subsection{Experiment 2: Varying Feature Dimension $d$}

\textbf{Instance Construction.} We fix $K = 64$ arms and budget $T = 1500$, and vary 
$d \in \{3, 6, 9, 18, 24\}$. For each $d$, the true 
parameter $\theta^* \in \mathbb{R}^d$ is a fixed dense 
unit vector whose coordinates are listed in 
Table~\ref{tab:vary_d_params}. The true feature of arm 
$i$ is constructed as:
\[
x_i^* = \mu_i \theta^* + \rho(d) z_i,
\]
where $z_i \perp \theta^*$ is a random unit vector 
orthogonal to $\theta^*$ and $\rho(d)$ controls the 
magnitude of the nuisance direction. The optimal arm is 
arm $0$ with true mean $\mu_0 = 0.62$ across all values 
of $d$. The second-best arm has true mean $\mu_1$, and 
the suboptimality gap $\Delta_1 = \mu_0 - \mu_1$ is 
denoted as ``gap'' in Table~\ref{tab:vary_d_params}; 
the remaining $62$ arms have means strictly decreasing 
from $\mu_1$ to $0.12$. The reward noise follows 
$\eta_t \sim \mathcal{N}(0, \xi^2)$, where $\xi$ varies 
with $d$ to maintain a consistent signal-to-noise ratio. 
All instance parameters are summarized in 
Table~\ref{tab:vary_d_params} and \ref{tab:vary_d_theta}.

% 参数表（gap, xi, rho）
\begin{table}[h]
\centering
\caption{Instance parameters for the varying-$d$ 
experiment ($K=64$, $T=1500$). ``gap'' denotes the 
suboptimality gap $\Delta_1 = \mu_0 - \mu_1$ between 
the optimal arm and the second-best arm.}
\label{tab:vary_d_params}
\renewcommand{\arraystretch}{1.3}
\begin{tabular}{c|c|c|c}
\hline
$d$ & gap & $\xi$ & $\rho(d)$ \\
\hline
3  & 0.0150 & 0.1400 & 0.350 \\
6  & 0.0135 & 0.1461 & 0.411 \\
9  & 0.0122 & 0.1522 & 0.472 \\
18 & 0.0096 & 0.1705 & 0.567 \\
24 & 0.0093 & 0.1826 & 0.567 \\
\hline
\end{tabular}
\end{table}

% theta* 表
\begin{table}[h]
\centering
\caption{Values of $\theta^*$ for each coordinate $k$ 
across different feature dimensions $d$ in the 
varying-$d$ experiment. ``---'' indicates the coordinate 
does not exist for the corresponding $d$.}
\label{tab:vary_d_theta}
\renewcommand{\arraystretch}{1.2}
\begin{tabular}{c|ccccc}
\hline
$k$ & $d=3$ & $d=6$ & $d=9$ & $d=18$ & $d=24$ \\
\hline
1  & 0.7385 & 0.6389 & 0.5945 & 0.5349 & 0.5146 \\
2  & 0.5222 & 0.4518 & 0.4204 & 0.3782 & 0.3639 \\
3  & 0.4264 & 0.3689 & 0.3433 & 0.3088 & 0.2971 \\
4  & ---    & 0.3194 & 0.2973 & 0.2674 & 0.2573 \\
5  & ---    & 0.2857 & 0.2659 & 0.2392 & 0.2301 \\
6  & ---    & 0.2608 & 0.2427 & 0.2184 & 0.2101 \\
7  & ---    & ---    & 0.2247 & 0.2022 & 0.1945 \\
8  & ---    & ---    & 0.2102 & 0.1891 & 0.1819 \\
9  & ---    & ---    & 0.1982 & 0.1783 & 0.1715 \\
10 & ---    & ---    & ---    & 0.1691 & 0.1627 \\
11 & ---    & ---    & ---    & 0.1613 & 0.1552 \\
12 & ---    & ---    & ---    & 0.1544 & 0.1486 \\
13 & ---    & ---    & ---    & 0.1484 & 0.1427 \\
14 & ---    & ---    & ---    & 0.1430 & 0.1375 \\
15 & ---    & ---    & ---    & 0.1381 & 0.1329 \\
16 & ---    & ---    & ---    & 0.1337 & 0.1287 \\
17 & ---    & ---    & ---    & 0.1297 & 0.1248 \\
18 & ---    & ---    & ---    & 0.1261 & 0.1213 \\
19 & ---    & ---    & ---    & ---    & 0.1181 \\
20 & ---    & ---    & ---    & ---    & 0.1151 \\
21 & ---    & ---    & ---    & ---    & 0.1123 \\
22 & ---    & ---    & ---    & ---    & 0.1097 \\
23 & ---    & ---    & ---    & ---    & 0.1073 \\
24 & ---    & ---    & ---    & ---    & 0.1050 \\
\hline
\end{tabular}
\end{table}
\textbf{Strategic Reporting.} 
Each arm $i$ acts as a self-interested agent aiming to 
maximize its own probability of being identified as the 
best arm. Under this strategic instance, every arm sets 
the first coordinate of its reported feature to zero, 
hiding the reward-relevant direction from the learner. 
The suboptimal arms report features with norm $1.0$ to 
appear as prominent as possible in the reported feature 
space, while the optimal arm's reported feature has norm 
$0.88$, making it appear less competitive than the 
suboptimal arms. Each arm $i$ maintains a pseudo-parameter 
$\theta_i^*$ satisfying $\langle x_{t,i}, \theta_i^*
\rangle = \mu_i$ for all $t$, ensuring that its 
estimated rewards under $\theta_i^*$ remain consistent 
with its observed rewards and the GTC check is passed. 
All results are averaged over 1000 independent trials.

\textbf{Baselines.} We compare MESHA against OptGTM 
\cite{kleine2024strategic} and OD-LinBAI*\cite{yang2022minimax}, which serves as an SOTA 
baseline facing non-strategic environment. Note that 
OD-LinBAI under strategic reporting fails completely in 
this instance and is omitted from this comparison for 
clarity.

\textbf{Results.} Algorithms' empirical BAI failure probabilities are summarized in 
Table \ref{tab:vary_d_results}.

\begin{table}[h]
\centering
\caption{BAI failure probabilities in the varying-$d$ experiment
($K=64$, $T=1500$)}
\label{tab:vary_d_results}
\renewcommand{\arraystretch}{1.3}
\begin{tabular}{c|c|c|c}
\hline
$d$ & MESHA & OptGTM & OD-LinBAI* \\
\hline
3  & 0.201 & 0.320 & 0.060 \\
6  & 0.236 & 0.358 & 0.125 \\
9  & 0.268 & 0.402 & 0.213 \\
18 & 0.337 & 0.442 & 0.314 \\
24 & 0.358 & 0.457 & 0.369 \\
\hline
\end{tabular}
\end{table}

\subsection{Experiment 3: Varying Arm Count $K$}

\textbf{Instance Construction.} We fix feature dimension 
$d = 6$, budget $T = 1500$, and sweep $K \in \{16, 24, 
32, 64\}$. The true parameter is the same dense unit 
vector $\theta^* = (0.6389,\ 0.4518,\ 0.3689,\ 0.3194,\ 
0.2857,\ 0.2608) \in \mathbb{R}^6$ across all values of 
$K$. The true feature of arm $i$ is constructed as:
\[
x_i^* = \mu_i \theta^* + \rho(K) z_i, \quad z_i \perp 
\theta^*, \quad \|z_i\|_2 = 1,
\]
where $z_i$ is a random unit vector orthogonal to 
$\theta^*$, and $\rho(K)$ controls the magnitude of the 
nuisance direction. The optimal arm is arm $0$ with true 
mean $\mu_0 = 0.62$ fixed across all $K$. The second-best 
arm has true mean $\mu_1 = \mu_0 - \text{gap}(K)$, where 
the suboptimality gap $\Delta_1 = \text{gap}(K)$ shrinks 
as $K$ grows according to:
\[
\text{gap}(K) = \max\!\left(0.0022,\ 0.010 \times 
\left(\frac{16}{\max(K,16)}\right)^{1.48}\right).
\]
The remaining $K-2$ arms have means linearly spaced 
between $\mu_1$ and $\mu_{\text{tail}} = 0.20$. The noise 
standard deviation $\xi$ and nuisance magnitude $\rho(K)$ 
increase with $K$ to reflect the greater difficulty of 
identifying the optimal arm when more arms compete. All 
instance parameters are summarized in 
Table~\ref{tab:vary_k_params}.

\begin{table}[h]
\centering
\caption{Instance parameters for the varying-$K$ 
experiment ($d=6$, $T=1500$). ``gap'' denotes the 
suboptimality gap $\Delta_1 = \mu_0 - \mu_1$ between 
the optimal arm and the second-best arm.}
\label{tab:vary_k_params}
\renewcommand{\arraystretch}{1.3}
\begin{tabular}{c|c|c|c|c}
\hline
$K$ & gap & $\xi$ & $\rho(K)$ & $\mu_1$ \\
\hline
16 & 0.0100 & 0.1400 & 0.350 & 0.6100 \\
24 & 0.0055 & 0.1512 & 0.376 & 0.6145 \\
32 & 0.0036 & 0.1675 & 0.412 & 0.6164 \\
64 & 0.0022 & 0.2128 & 0.497 & 0.6178 \\
\hline
\end{tabular}
\end{table}

\textbf{Strategic Reporting.} Under strategic reporting, 
every arm sets the first coordinate of its reported feature 
to zero. Unlike the vary-$d$ experiment, the true optimal 
arm remains the strongest arm under reported features as 
well (no collapse bait), so the strategic difficulty comes 
purely from the geometry compression rather than active 
misdirection. Each arm maintains a pseudo-parameter 
satisfying $\langle x_{t,i}, \theta_i^*\rangle = 
\mu_i$ to pass the GTC. The reported geometry is scaled by 
a $K$-dependent multiplier (K=16$\to$1.00, 24$\to$1.05, 
32$\to$1.12, 64$\to$1.51) to reflect increasing strategic 
pressure as more arms compete. To ensure consistency across 
$K$ values, the instance uses a nested construction where 
the first $K'$ arms under a larger $K$ are identical to 
those under a smaller $K$. All results are averaged over 
5000 independent trials.

\textbf{Baselines.} We compare MESHA against Sequential 
Halving \cite{karnin2013almost} and SE, both of which rely solely on 
observed rewards and are unaffected by strategic feature 
reporting. OD-LinBAI and OptGTM are omitted from this 
comparison since they fail under strategic reporting in 
this instance.

\textbf{Results.} The BAI failure probabilities are summarized in Table 
\ref{tab:vary_k_results}.

\begin{table}[h]
\centering
\caption{BAI failure probabilities for varying $K$ 
($d=6$, $T=1500$, $\pm$ indicates Wald $95\%$ CI half-width).}
\label{tab:vary_k_results}
\renewcommand{\arraystretch}{1.3}
\begin{tabular}{c|c|c|c}
\hline
$K$ & MESHA & Sequential Halving & SR \\
\hline
16 & $0.222$ & $0.247 $ & 
$0.2144 $ \\
24 & $0.342 $ & $0.397 $ & 
$0.3524 $ \\
32 & $0.390 $ & $0.415$ & 
$0.4198 $ \\
64 & $0.426 $ & $0.482 $ & 
$0.4716 $ \\
\hline
\end{tabular}
\end{table}


\bibliographystyle{IEEEtran}
\bibliography{references}

@article{chapelle2011empirical,
  title={An empirical evaluation of thompson sampling},
  author={Chapelle, Olivier and Li, Lihong},
  journal={Advances in neural information processing systems},
  volume={24},
  year={2011}
}

@article{yang2022minimax,
  title={Minimax optimal fixed-budget best arm identification in linear bandits},
  author={Yang, Junwen and Tan, Vincent},
  journal={Advances in Neural Information Processing Systems},
  volume={35},
  pages={12253--12266},
  year={2022}
}

@article{soare2014best,
  title={Best-arm identification in linear bandits},
  author={Soare, Marta and Lazaric, Alessandro and Munos, R{\'e}mi},
  journal={Advances in neural information processing systems},
  volume={27},
  year={2014}
}

@inproceedings{xu2018fully,
  title={A fully adaptive algorithm for pure exploration in linear bandits},
  author={Xu, Liyuan and Honda, Junya and Sugiyama, Masashi},
  booktitle={International Conference on Artificial Intelligence and Statistics},
  pages={843--851},
  year={2018},
  organization={PMLR}
}

@inproceedings{audibert2010best,
  title={Best arm identification in multi-armed bandits},
  author={Audibert, Jean-Yves and Bubeck, S{\'e}bastien},
  booktitle={COLT-23th Conference on learning theory-2010},
  pages={13--p},
  year={2010}
}

@article{auer2002finite,
  title={Finite-time analysis of the multiarmed bandit problem},
  author={Auer, Peter and Cesa-Bianchi, Nicolo and Fischer, Paul},
  journal={Machine learning},
  volume={47},
  number={2},
  pages={235--256},
  year={2002},
  publisher={Springer}
}

@inproceedings{karnin2013almost,
  title={Almost optimal exploration in multi-armed bandits},
  author={Karnin, Zohar and Koren, Tomer and Somekh, Oren},
  booktitle={International conference on machine learning},
  pages={1238--1246},
  year={2013},
  organization={PMLR}
}

@article{shen2019universal,
  title={Universal best arm identification},
  author={Shen, Cong},
  journal={IEEE Transactions on Signal Processing},
  volume={67},
  number={17},
  pages={4464--4478},
  year={2019},
  publisher={IEEE}
}

@article{kleine2024strategic,
  title={Strategic linear contextual bandits},
  author={Kleine Buening, Thomas and Saha, Aadirupa and Dimitrakakis, Christos and Xu, Haifeng},
  journal={Advances in Neural Information Processing Systems},
  volume={37},
  pages={116638--116675},
  year={2024}
}

@inproceedings{braverman2019multi,
  title={Multi-armed bandit problems with strategic arms},
  author={Braverman, Mark and Mao, Jieming and Schneider, Jon and Weinberg, S Matthew},
  booktitle={Conference on Learning Theory},
  pages={383--416},
  year={2019},
  organization={PMLR}
}

@article{even2006action,
  title={Action elimination and stopping conditions for the multi-armed bandit and reinforcement learning problems.},
  author={Even-Dar, Eyal and Mannor, Shie and Mansour, Yishay and Mahadevan, Sridhar},
  journal={Journal of machine learning research},
  volume={7},
  number={6},
  year={2006}
}

@article{gabillon2012best,
  title={Best arm identification: A unified approach to fixed budget and fixed confidence},
  author={Gabillon, Victor and Ghavamzadeh, Mohammad and Lazaric, Alessandro},
  journal={Advances in neural information processing systems},
  volume={25},
  year={2012}
}

@article{abbasi2011improved,
  title={Improved algorithms for linear stochastic bandits},
  author={Abbasi-Yadkori, Yasin and P{\'a}l, D{\'a}vid and Szepesv{\'a}ri, Csaba},
  journal={Advances in neural information processing systems},
  volume={24},
  year={2011}
}

@inproceedings{esmaeili2025robust,
  title={Robust performance incentivizing algorithms for multi-armed bandits with strategic agents},
  author={Esmaeili, Seyed A and Shin, Suho and Slivkins, Aleksandrs},
  booktitle={Proceedings of the AAAI Conference on Artificial Intelligence},
  volume={39},
  number={13},
  pages={13814--13822},
  year={2025}
}

@article{verma2025cobra,
  title={COBRA: Contextual Bandit Algorithm for Ensuring Truthful Strategic Agents},
  author={Verma, Arun and Saha, Indrajit and Yokoo, Makoto and Low, Bryan Kian Hsiang},
  journal={arXiv preprint arXiv:2505.23720},
  year={2025}
}

@article{esmaeili2023replication,
  title={Replication-proof bandit mechanism design},
  author={Esmaeili, Seyed and Hajiaghayi, MohammadTaghi and Shin, Suho},
  journal={arXiv e-prints},
  pages={arXiv--2312},
  year={2023}
}

@inproceedings{shin2022multi,
  title={Multi-armed bandit algorithm against strategic replication},
  author={Shin, Suho and Lee, Seungjoon and Ok, Jungseul},
  booktitle={International Conference on Artificial Intelligence and Statistics},
  pages={403--431},
  year={2022},
  organization={PMLR}
}

@inproceedings{feng2020intrinsic,
  title={The intrinsic robustness of stochastic bandits to strategic manipulation},
  author={Feng, Zhe and Parkes, David and Xu, Haifeng},
  booktitle={International Conference on Machine Learning},
  pages={3092--3101},
  year={2020},
  organization={PMLR}
}

@article{kaufmann2016complexity,
  title={On the complexity of best-arm identification in multi-armed bandit models},
  author={Kaufmann, Emilie and Capp{\'e}, Olivier and Garivier, Aur{\'e}lien},
  journal={The Journal of Machine Learning Research},
  volume={17},
  number={1},
  pages={1--42},
  year={2016},
  publisher={JMLR. org}
}

@article{fiez2019sequential,
  title={Sequential experimental design for transductive linear bandits},
  author={Fiez, Tanner and Jain, Lalit and Jamieson, Kevin G and Ratliff, Lillian},
  journal={Advances in neural information processing systems},
  volume={32},
  year={2019}
}

@article{lai1985asymptotically,
  title={Asymptotically efficient adaptive allocation rules},
  author={Lai, Tze Leung and Robbins, Herbert},
  journal={Advances in applied mathematics},
  volume={6},
  number={1},
  pages={4--22},
  year={1985},
  publisher={Academic Press, Inc. Orlando, FL, USA}
}

@article{zheng2022ai,
  title={The AI Economist: Taxation policy design via two-level deep multiagent reinforcement learning},
  author={Zheng, Stephan and Trott, Alexander and Srinivasa, Sunil and Parkes, David C and Socher, Richard},
  journal={Science advances},
  volume={8},
  number={18},
  pages={eabk2607},
  year={2022},
  publisher={American Association for the Advancement of Science}
}

@inproceedings{dani2008stochastic,
  title={Stochastic linear optimization under bandit feedback},
  author={Dani, Varsha and Hayes, Thomas P and Kakade, Sham M},
  booktitle={21st Annual Conference on Learning Theory},
  number={101},
  pages={355--366},
  year={2008}
}

@book{lattimore2020bandit,
  title={Bandit algorithms},
  author={Lattimore, Tor and Szepesv{\'a}ri, Csaba},
  year={2020},
  publisher={Cambridge University Press}
}

@inproceedings{even2002pac,
  title={PAC bounds for multi-armed bandit and Markov decision processes},
  author={Even-Dar, Eyal and Mannor, Shie and Mansour, Yishay},
  booktitle={International Conference on Computational Learning Theory},
  pages={255--270},
  year={2002},
  organization={Springer}
}

@article{thompson1933likelihood,
  title={On the likelihood that one unknown probability exceeds another in view of the evidence of two samples},
  author={Thompson, William R},
  journal={Biometrika},
  volume={25},
  number={3/4},
  pages={285--294},
  year={1933},
  publisher={JSTOR}
}

@inproceedings{agrawal2012analysis,
  title={Analysis of thompson sampling for the multi-armed bandit problem},
  author={Agrawal, Shipra and Goyal, Navin},
  booktitle={Conference on learning theory},
  pages={39--1},
  year={2012},
  organization={JMLR Workshop and Conference Proceedings}
}

@article{auer2002using,
  title={Using confidence bounds for exploitation-exploration trade-offs},
  author={Auer, Peter},
  journal={Journal of machine learning research},
  volume={3},
  number={Nov},
  pages={397--422},
  year={2002}
}

@article{zhong2021achieving,
  title={Achieving the pareto frontier of regret minimization and best arm identification in multi-armed bandits},
  author={Zhong, Zixin and Cheung, Wang Chi and Tan, Vincent YF},
  journal={arXiv preprint arXiv:2110.08627},
  year={2021}
}

@inproceedings{degenne2019bridging,
  title={Bridging the gap between regret minimization and best arm identification, with application to a/b tests},
  author={Degenne, R{\'e}my and Nedelec, Thomas and Calauz{\`e}nes, Cl{\'e}ment and Perchet, Vianney},
  booktitle={The 22nd International Conference on Artificial Intelligence and Statistics},
  pages={1988--1996},
  year={2019},
  organization={PMLR}
}

@article{hoeffding1963probability,
  title={Probability inequalities for sums of bounded random variables},
  author={Hoeffding, Wassily},
  journal={Journal of the American statistical association},
  volume={58},
  number={301},
  pages={13--30},
  year={1963},
  publisher={Taylor \& Francis}
}

@inproceedings{zhao2023revisiting,
  title={Revisiting simple regret: Fast rates for returning a good arm},
  author={Zhao, Yao and Stephens, Connor and Szepesv{\'a}ri, Csaba and Jun, Kwang-Sung},
  booktitle={International Conference on Machine Learning},
  pages={42110--42158},
  year={2023},
  organization={PMLR}
}

\vfill

\end{document}